\documentclass{article}
\newcommand{\MedicalHarnessArxiv}{}
\ifdefined\MedicalHarnessArxiv
\else
  \documentclass{article}
\fi

\PassOptionsToPackage{table}{xcolor}
\usepackage[T1]{fontenc}
\ifdefined\MedicalHarnessArxiv
  \usepackage{arxiv,times}
\else
  \usepackage{iclr2027_conference,times}
\fi

\usepackage{dsfont} 

\usepackage{amsmath,amsfonts,bm}

\def\eqref#1{equation~\ref{#1}}

\def\1{\bm{1}}

\DeclareMathAlphabet{\mathsfit}{\encodingdefault}{\sfdefault}{m}{sl}
\SetMathAlphabet{\mathsfit}{bold}{\encodingdefault}{\sfdefault}{bx}{n}

\usepackage{amsmath}
\usepackage{amssymb}
\usepackage{amsfonts}
\usepackage{amsthm}
\usepackage{booktabs}
\usepackage{multirow}
\usepackage{makecell}
\usepackage{tabularx}
\usepackage{graphicx}
\usepackage{subcaption}
\usepackage{algorithm}
\usepackage{algpseudocode}
\usepackage[table]{xcolor}
\usepackage{xspace}
\usepackage{pifont}
\usepackage{enumitem}
\usepackage{microtype}
\usepackage{url}
\usepackage{hyperref}
\ifdefined\MedicalHarnessArxiv\usepackage{fontawesome5}\fi
\usepackage{cleveref}
\usepackage{wrapfig}
\usepackage{placeins}

\usepackage[most]{tcolorbox}

\newtheoremstyle{mhresult}{0pt}{0pt}{\itshape}{0pt}{\normalfont\bfseries}{.}{5pt plus 1pt minus 1pt}
  {\thmname{#1}\thmnumber{ #2}\thmnote{ {\mdseries(#3)}}}
\newtheoremstyle{mhdef}{0pt}{0pt}{\normalfont}{0pt}{\normalfont\bfseries}{.}{5pt plus 1pt minus 1pt}
  {\thmname{#1}\thmnumber{ #2}\thmnote{ {\mdseries(#3)}}}
\theoremstyle{plain}

\theoremstyle{mhresult}

\newtheorem{proposition}{Proposition}[section]

\theoremstyle{mhdef}
\newtheorem{definition}{Definition}[section]

\tcbset{resultbox/.style={enhanced, breakable, sharp corners, frame hidden, boxrule=0pt, colback=black!4,
  boxsep=0pt, left=8pt, right=8pt, top=6pt, bottom=6pt, before skip=9pt, after skip=7pt}}
\tcolorboxenvironment{theorem}{resultbox, unbreakable}
\tcolorboxenvironment{proposition}{resultbox, unbreakable}
\tcolorboxenvironment{definition}{resultbox, unbreakable}
\tcolorboxenvironment{corollary}{resultbox, unbreakable, before skip=7pt, after skip=6pt}
\newcommand{\proofnote}[1]{\par\noindent\textit{Proof.} See #1.\par}

\definecolor{lightgray}{gray}{0.9}
\definecolor{lightergray}{gray}{0.95}

\definecolor{mhName}{HTML}{9B2C4E}
\definecolor{mhLabName}{HTML}{2F6DB5}
\newcommand{\mhname}[2]{\textcolor{#1}{{\fontfamily{lmtt}\fontshape{n}\selectfont\bfseries #2}}}
\newcommand{\methodname}{\mhname{mhName}{MedicalHarness}\xspace}
\newcommand{\benchsuite}{\mhname{mhName}{MedicalHarnessBench}\xspace}

\newcommand{\fone}{\ding{182}}
\newcommand{\ftwo}{\ding{183}}
\newcommand{\fthree}{\ding{184}}
\newcommand{\ffour}{\ding{185}}

\newcommand{\ctool}{\mbox{$\mathop{\mathtt{MedTool}}\limits$}\xspace}

\newcommand{\cmem}{\mbox{$\mathop{\mathtt{MedMemory}}\limits$}\xspace}

\newcommand{\cplan}{\mbox{$\mathop{\mathtt{MedPlanning}}\limits$}\xspace}

\newcommand{\cweb}{\mbox{$\mathop{\mathtt{MedWeb}}\limits$}\xspace}

\newcommand{\mqwen}{\mbox{$\mathop{\mathtt{Qwen3.6\text{-}35B\text{-}A3B}}\limits$}\xspace}
\newcommand{\mlightning}{\mbox{$\mathop{\mathtt{Nemotron\text{-}3.5\text{-}30B\text{-}A3B}}\limits$}\xspace}

\newcommand{\mgptoss}{\mbox{$\mathop{\mathtt{GPT\text{-}OSS\text{-}120B}}\limits$}\xspace}
\newcommand{\mgemma}{\mbox{$\mathop{\mathtt{Gemma4\text{-}26B\text{-}A4B}}\limits$}\xspace}
\newcommand{\lgharness}{\mhname{mhLabName}{MH-Lab}\xspace}   
\newcommand{\mglm}{\mbox{$\mathop{\mathtt{GLM\text{-}4.7\text{-}Flash}}\limits$}\xspace}

\definecolor{mhBlock}{HTML}{DDEBF7}   
\definecolor{mhOurs}{HTML}{F8E1E1}    
\definecolor{mhNavy}{HTML}{24527A}    
\definecolor{mhFaint}{HTML}{707070}   
\definecolor{mhRule}{HTML}{58606A}
\newcolumntype{I}{!{\color{mhRule}\vrule width 0.3pt}}
\newcolumntype{Y}{>{\centering\arraybackslash}X}

\newcommand{\meanstd}[2]{#1{\scriptsize\color{mhFaint}$\pm$#2}}

\newcommand{\firstresult}[1]{\textbf{#1}}
\newcommand{\secondresult}[1]{\underline{#1}}
\newcommand{\baseline}[1]{\texttt{#1}}

\title{\methodname: A Controlled Evaluation of LLMs and Agent Harnesses on Medical Tasks}
\ifdefined\MedicalHarnessArxiv
  \author{Ziqing Wang \quad Lili Zhao \quad Kaize Ding\\Northwestern University\\[5pt]
    \textnormal{\small\faGlobe\ \href{https://real-lab-nu.github.io/MedicalHarness/}{\textcolor{mhName}{\textbf{Project Page}}}\quad\faGithub\ \href{https://github.com/REAL-Lab-NU/MedicalHarness}{\textcolor{mhName}{\textbf{Code \& Data}}}}}
  \hypersetup{colorlinks=true, linkcolor=mhName, citecolor=mhLabName, urlcolor=mhName}
  \iclrfinalcopy
  \newcommand{\repositorylink}{\url{https://github.com/REAL-Lab-NU/MedicalHarness}}
  \hypersetup{
    pdftitle={MedicalHarness: A Controlled Evaluation of LLMs and Agent Harnesses on Medical Tasks},
    pdfauthor={Ziqing Wang, Lili Zhao, Kaize Ding}
  }
\else
  \author{Anonymous authors}
  \newcommand{\repositorylink}{\url{https://anonymous.4open.science/r/MedicalHarness/}}
\fi

\begin{document}

\maketitle
\suppressfloats[t]

\begin{abstract}
LLM agents are increasingly built for medical work and scored on clinical benchmarks. Each such score,
however, comes from a model running inside an agent harness, the system that controls the loop between the
model and its environment. An agent's score is therefore a property of a model--harness pair. For medical
agents, how much outcomes change with the harness has rarely been measured. Measuring this change, and
explaining it, raises two challenges. First, a harness comparison must change nothing but the harness and be
repeated across models and kinds of task. Second, comparing whole harnesses leaves their mechanisms bundled
together, so it cannot show when an individual mechanism helps. To address these challenges, we present
\methodname, a controlled study of models and agent harnesses on medical tasks. We first build \benchsuite
to evaluate agents on $107$ tasks across four domains that each test a different harness capability. Using
this benchmark, we run five open-weight models under five agent harnesses, changing only the harness within
a comparison, and analyze both outcomes and execution traces. To study individual mechanisms, we build
\lgharness, a controlled harness that switches off context management, planning or tool exposure one at a time
within a shared execution loop. We find that the harness and its interaction with the model account for
about a quarter of the outcome variance, and that no single harness is best across models and tasks.
Code and data are available at
\repositorylink.
\end{abstract}

\begin{figure}[t]
  \centering
  \includegraphics[width=\textwidth]{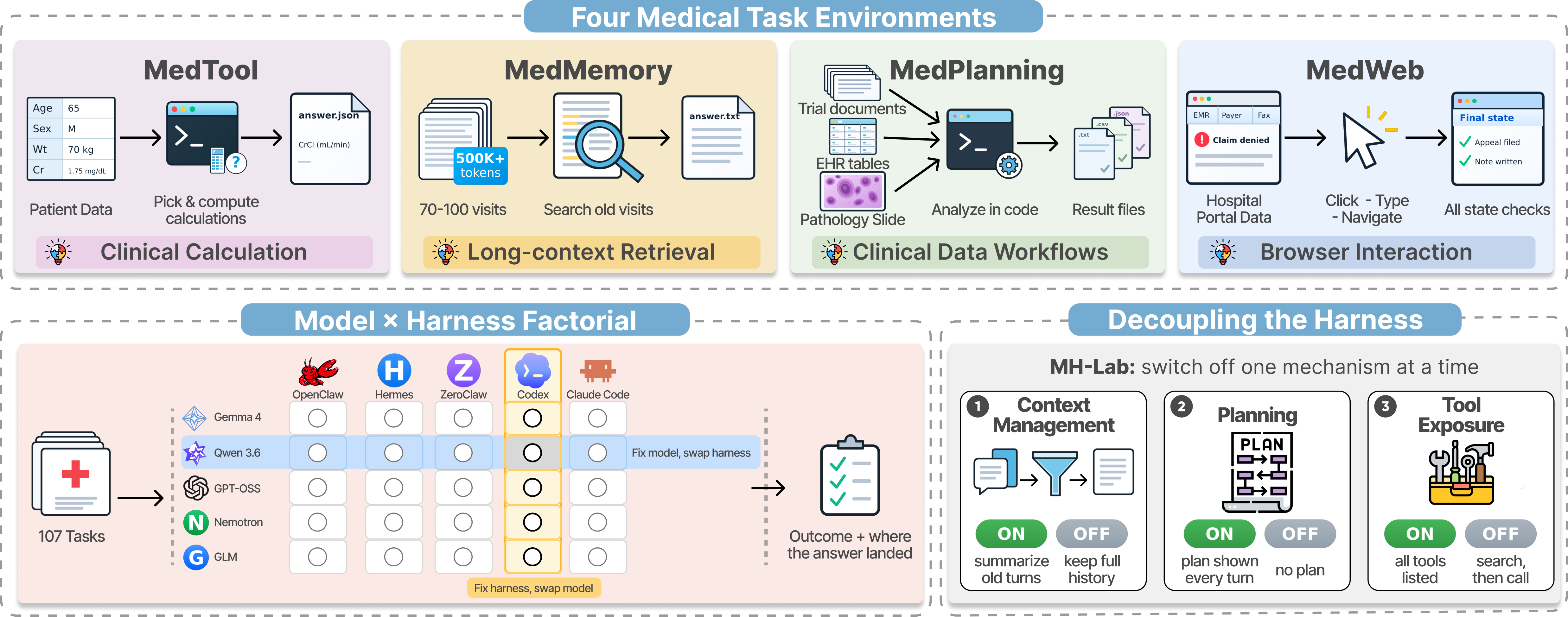}
  \captionsetup{skip=3pt}
  \caption{\small\textbf{Overview of \methodname.} \textbf{Top:} \benchsuite's four environments, each
  testing a different harness capability: \ctool (clinical calculation), \cmem (long-context
  retrieval), \cplan (clinical data workflows) and \cweb (browser interaction). \textbf{Bottom left:}
  every model runs under every harness on all $107$ tasks, so each row varies only the harness and
  each column only the model. \textbf{Bottom right:} \lgharness switches off one mechanism at a time, keeping the rest
  of its loop fixed: context management (summarizing old turns), planning (a plan shown every turn)
  or tool exposure (all tools listed instead of hidden behind search).}
  \label{fig:overview}
  \vspace{-8pt}
\end{figure}

\section{Introduction}
\label{sec:intro}

LLM agents have drawn growing interest in medicine, with the promise of supporting clinical work
and reducing clinicians' administrative burden~\citep{wang2025baymax,jiang2025medagentbench,wang2025amanda,wang2026latentdx}.
These applications require agents to act within clinical systems, from computing risk scores and
finding facts in patient records to checking trial eligibility and resolving insurance denials
across hospital portals. To assess how well agents complete such work, a growing set of benchmarks
evaluates them in simulated clinical
environments~\citep{jiang2025medagentbench,schmidgall2025agentclinic,liu2026healthagentbench,bedi2026healthadminbench,chen2026chibench,liu2026benchmarking}.
The resulting scores reflect both the model
and the agent system that guides its interactions with the environment. Understanding these results
therefore raises a basic question: \emph{how much do outcomes change when the same model runs
inside a different agent system?}

The system around the model is its agent harness. Harnesses such as Claude Code~\citep{claudecode},
Codex~\citep{codexcli}, OpenClaw~\citep{openclaw} and Hermes~\citep{hermesagent} determine what the
model sees, which tools it can call, how its actions are executed, how its history is managed and
when execution ends~\citep{li2026harness,wang2026openclawsurvey}. History management affects which clinical evidence
remains available~\citep{wang2026medmemorybench,zhang2026medlocomo}, while stopping rules determine whether an unfinished workflow can continue. Such
differences can lead the same model to different
outcomes~\citep{zhang2026disclosingharness,yao2026harnessbench}. An agent's score is therefore a
property of a model--harness pair. In medicine, however, few studies compare harnesses under
controlled conditions across models and kinds of task. Making such comparisons, and
explaining the differences they reveal, raises two challenges.

\textbf{Challenge 1: measuring and understanding harness differences.} When agent comparisons
change the model, serving conditions or execution budget along with the harness, the resulting
score gaps mix these effects~\citep{zhang2026disclosingharness,yao2026harnessbench}. A controlled
comparison must therefore run the same model on the same tasks, with the serving setup and budgets
held fixed. A result from one model on one task type may not hold
elsewhere, so this comparison must be repeated across models and different kinds of medical
task. Even under these controls, the same low score can reflect different failures, such as a
wrong answer, no answer at all or an unfinished workflow, so execution traces and task diagnostics
are needed to tell them apart.

\textbf{Challenge 2: understanding when individual mechanisms help.} Whole-harness comparisons
leave context management, planning and tool exposure bundled together, making their individual
effects difficult to identify~\citep{li2026harness}. Whether a mechanism helps can also depend on
the model and the task. Some mechanisms, such as summarization, operate only under particular
conditions, so enabling them does not ensure that they are used. Studying their effects therefore
requires changing one mechanism at a time within a shared execution loop, with the model, task,
underlying tools and budget held fixed, and in a setting where the mechanism actually fires.

To address these challenges, we present \methodname, a controlled study of models and agent
harnesses on medical tasks (\Cref{fig:overview}). We first build \benchsuite to evaluate agents
across four capability domains. Using this benchmark, we compare model--harness pairs under
matched execution conditions and analyze their outcomes and execution traces. To study individual
mechanisms, we build \lgharness, which switches off context management, planning or tool exposure
one at a time within a shared execution loop. We find that the harness and its interaction with the
model account for a substantial share of outcome variance, and that no single harness is best
across models and tasks. Harnesses also usually differ more in how often they answer than in their scores on commonly
answered questions, and hiding tools behind a search interface, meant to save context, can instead
cost tokens or accuracy. Our contributions are:
\begin{itemize}[leftmargin=*,noitemsep,topsep=2pt]
  \item \textbf{Benchmark.} We introduce \benchsuite, a benchmark of $107$ executable medical tasks
  spanning clinical calculation, long-context retrieval, clinical data workflows and browser
  interaction, four domains that each test a different harness capability.
  \item \textbf{Controlled comparison study.} We evaluate five open-weight models under five
  harnesses with three seeds ($8{,}025$ episodes), changing only the harness within a comparison.
  We pair a model--harness variance analysis with answer-production and execution diagnostics.
  \item \textbf{Controlled mechanism study.} We build \lgharness to test context management,
  planning and tool exposure through one-at-a-time interventions within a shared execution loop.
\end{itemize}

\section{Related Work}
\label{sec:related_work}

\textbf{Medical agent benchmarks.}
MedAgentBench evaluates task execution in a virtual EHR~\citep{jiang2025medagentbench}, and
AgentClinic simulates interactive clinical decision-making~\citep{schmidgall2025agentclinic}.
HealthAgentBench assembles executable environments across clinical data, imaging and trial
matching~\citep{liu2026healthagentbench}, HealthAdminBench and $\chi$-Bench evaluate agents on
healthcare administration and managed-care workflows~\citep{bedi2026healthadminbench,chen2026chibench},
MedMCP-Calc exposes medical calculators to agents through MCP~\citep{zhu2026medmcpcalc}, and
MedMemoryBench tests memory over long patient histories~\citep{wang2026medmemorybench}. These
benchmarks measure what an agent can do in a medical environment. HealthAgentBench and $\chi$-Bench also
compare harnesses, but \methodname is dedicated to this comparison and makes it more comprehensive
and better controlled. It runs every model under every harness on more tasks, spanning four kinds
that each test a different harness capability, serves all models locally under one setup so that only
the harness changes, and adds controlled experiments on individual mechanisms.

\textbf{Harnesses and their mechanisms.}
A harness, also called a scaffold, changes what a fixed model can
do~\citep{yang2024sweagent,wang2025openhands}, yet general agent benchmarks such as AgentBench and
WebArena are built mainly to compare models~\citep{liu2024agentbench,zhou2024webarena}. Recent work
evaluates the harness itself. Harness-Bench measures harness effects across
models~\citep{yao2026harnessbench}, \citet{zhang2026disclosingharness} model the harness as the
controller of a closed loop around the model and propose disclosure standards with a variance
decomposition, \citet{li2026harness} organize harness components into seven families, and
\citet{fan2026harness} ablate components for coding agents. Controlled comparisons outside medicine
find that harness effects vary with the task and the model, from large cost but small accuracy
differences on terminal coding~\citep{vats2026scaffold} to accuracy gaps of up to $28$ points on
GAIA~\citep{starace2026gaia}, and \citet{zhangy2026confound} show that a score reflects the model only
if the harness leaves the decisions to it. \methodname also separates how often agents answer from their scores on commonly answered questions, and uses
\lgharness to switch off three mechanisms that product harnesses share, context
management~\citep{packer2023memgpt}, planning and tool exposure~\citep{qin2024toolllm}, one at a time
within one execution loop.

\section{Problem Formulation}
\label{sec:problem_definition}

We first define what a harness controls and what we hold fixed when comparing harnesses
(\Cref{sec:formal-harness}). We then separate variation in benchmark scores into model, harness,
interaction and repeated-run components (\Cref{sec:formal-variance}).

\subsection{Agent harness}
\label{sec:formal-harness}

A harness decides what the model sees, how its actions are executed, and when the episode ends.

\begin{definition}[Agent harness]
\label{def:agent}
An agent harness is a triple of procedures $h = (C_h, U_h, T_h)$ that runs a model $m$ on a task $x$.
From the task's initial state $s_0$, while $T_h(s_t) = 0$ it repeats
\[
c_t = C_h(s_t), \qquad a_t \sim \pi_m(\cdot \mid c_t), \qquad
(o_t, s_{t+1}) \sim U_h(\cdot \mid s_t, a_t; x).
\]
Here $\pi_m(\cdot \mid c_t)$ is the model's response distribution.
$C_h$ builds the context, $U_h$ executes the action in the task environment, collects the observation
$o_t$ and updates the state, and $T_h$ decides when to stop. The state contains the history, files,
plan, environment state and remaining budget. An agent is the pair $(m, h)$, and the task's oracle
scores its trajectory $\tau$ as $Y = R_x(\tau) \in [0,1]$.
\end{definition}

$C_h$ selects the instructions, tool descriptions and history shown to the model.
State updates in $U_h$ include tool routing, output checks and any auxiliary model
calls, such as summarization, that prepare the next turn. In our study these calls use the same
model $\pi_m$, are recorded in the trajectory and count toward the episode's budget.

This view guides both comparisons. When swapping harnesses, we hold the tasks and environments,
model checkpoint, serving stack, context window, episode deadline and run seed set fixed.
Each harness runs as released, so its prompts, tool formats, retry and stopping rules, output limits and
sampling fields belong to the harness being compared.
When removing a mechanism from \lgharness, we change the procedures that
implement it and keep the rest fixed (\Cref{sec:components}). Context management and planning change
what $C_h$ builds and $U_h$ keeps. Tool exposure changes both the tool descriptions in $C_h$ and
the routing and feedback in $U_h$.

\subsection{Model--harness decomposition}
\label{sec:formal-variance}

To measure how much the harness matters, we split variation in benchmark scores into four parts:
the model, the harness, their interaction, and variation across repeated runs. The interaction
captures whether a harness suits some models better than others. Repeated runs of the same
model--harness pair can also produce different scores. We report this variation separately.

\begin{definition}[Model--harness decomposition]
\label{def:decomp}
Within an environment, let every model $m \in \mathcal{M}$ run under every harness $h \in \mathcal{H}$
with every seed $\sigma \in \mathcal{S}$, and let $\bar y_{mh\sigma}$ be the mean outcome over the
environment's tasks. We decompose
\begin{equation}
\bar y_{mh\sigma} = \mu + \alpha_m + \beta_h + \gamma_{mh} + \varepsilon_{mh\sigma}.
\label{eq:anova}
\end{equation}
Here $\mu$ is the grand mean. The model effect $\alpha_m$ is the model's mean score across harnesses
and seeds minus $\mu$. The harness effect $\beta_h$ is the harness's mean across models and seeds
minus $\mu$.
The interaction $\gamma_{mh}$ is the pair's mean score minus $(\mu+\alpha_m+\beta_h)$, and
$\varepsilon_{mh\sigma}$ is a run's deviation from its pair's mean.
\end{definition}

The total score variation is $\mathrm{SS}_T = \sum_{m,h,\sigma}(\bar y_{mh\sigma}-\mu)^2$.
Summing each component's squared values over all $(m,h,\sigma)$ gives $\mathrm{SS}_M$,
$\mathrm{SS}_H$, $\mathrm{SS}_{MH}$ and $\mathrm{SS}_E$, respectively.
The cross terms cancel in this balanced design, giving
\begin{equation}
\mathrm{SS}_T = \mathrm{SS}_M + \mathrm{SS}_H + \mathrm{SS}_{MH} + \mathrm{SS}_E,
\label{eq:ss}
\end{equation}
with the expanded formulas in \Cref{sec:appendix-variance}.

When $\mathrm{SS}_T > 0$, we report each component as a share of the total score variation.
The \emph{harness share}, $(\mathrm{SS}_H + \mathrm{SS}_{MH})/\mathrm{SS}_T$, combines the harness
effect and the model--harness interaction. The share $\mathrm{SS}_E/\mathrm{SS}_T$ measures
run-to-run variation.

We also ask whether scores vary more when changing the harness under a fixed model or changing
the model under a fixed harness. We measure these as harness variance (HV) and model variance
(MV)~\citep{zhang2026disclosingharness}:
$\mathrm{HV}(m) = \operatorname{Var}_h(\bar y_{mh\cdot})$ and
$\mathrm{MV}(h) = \operatorname{Var}_m(\bar y_{mh\cdot})$, where the dot averages over seeds.
We use population variances over the selected harnesses or models, dividing by $|\mathcal{H}|$ or
$|\mathcal{M}|$, respectively.

\begin{proposition}[Harness and model variance]
\label{prop:hvmv}
In a balanced model $\times$ harness $\times$ seed design, let $N = |\mathcal{M}||\mathcal{H}||\mathcal{S}|$.
Averaging $\mathrm{HV}$ over models and $\mathrm{MV}$ over harnesses gives
\[
\overline{\mathrm{HV}} = \frac{\mathrm{SS}_H + \mathrm{SS}_{MH}}{N},
\qquad
\overline{\mathrm{MV}} = \frac{\mathrm{SS}_M + \mathrm{SS}_{MH}}{N}.
\]
Consequently, $\overline{\mathrm{HV}} > \overline{\mathrm{MV}}$ if and only if
$\mathrm{SS}_H > \mathrm{SS}_M$.
\end{proposition}

\proofnote{\Cref{sec:appendix-variance}}

The interaction enters both averages. Which average is larger therefore depends only on the two
main effects. The decomposition reports the interaction separately.

\section{\benchsuite}
\label{sec:method}

\begin{figure}[!t]
  \centering
  \includegraphics[width=\textwidth]{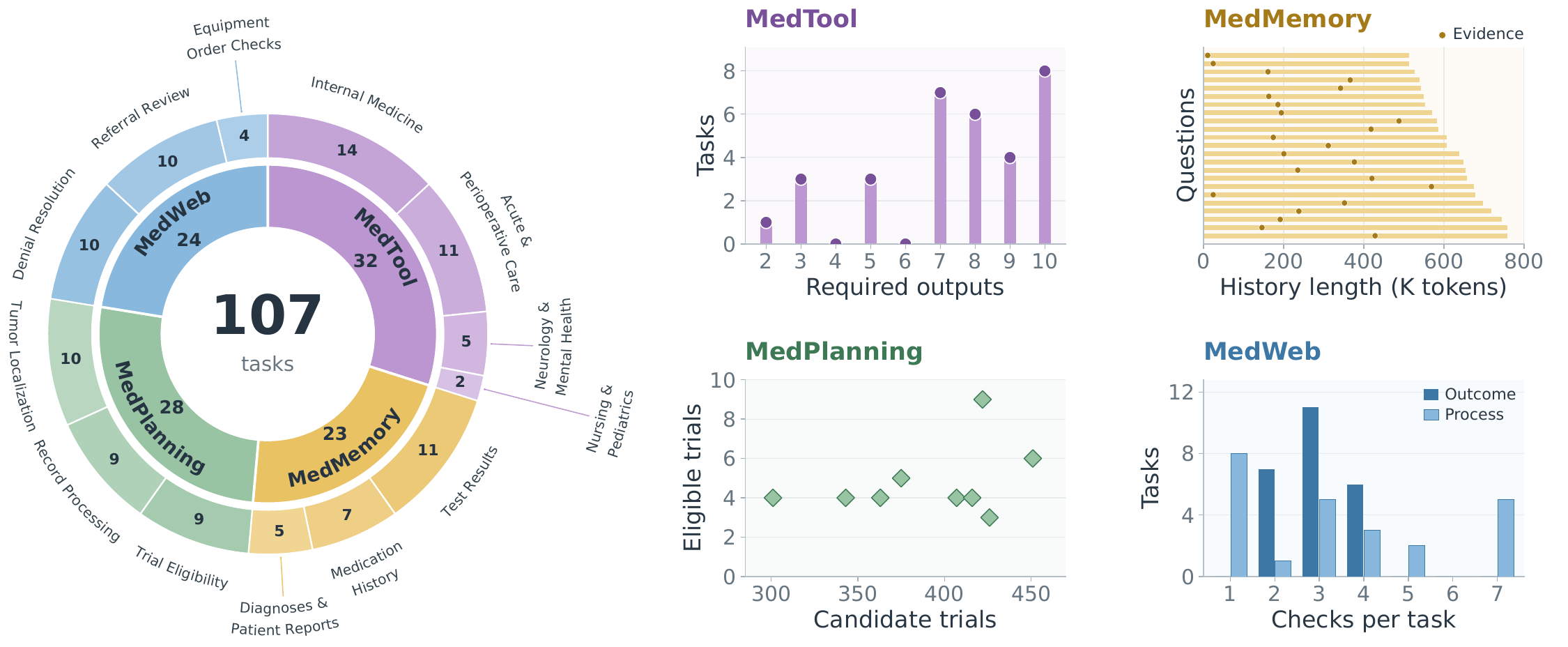}
  \caption{\small\textbf{Task composition and characteristics.} \textbf{Left:} the $107$ task instances
  grouped by clinical domain or workflow (\Cref{tab:task-composition}). \textbf{Right:} required \ctool
  outputs, one \cmem history per bar with source-evidence session starts marked by dots
  (Qwen3.6 tokens), candidate and eligible trials for nine \cplan trial-eligibility tasks, and
  separate distributions of \cweb outcome and process check counts.}
  \label{fig:task-composition}
  \vspace{-6pt}
\end{figure}

\subsection{Environments}
\label{sec:environments}

\benchsuite covers four capability domains: clinical calculation, long-context retrieval,
clinical data workflows and browser interaction. Its $107$ tasks span clinical information
processing and hospital administration (\Cref{fig:task-composition}).

\textbf{Clinical calculation (\ctool, $32$ tasks).} Each task provides a clinician's request and
a patient's record stored in a SQLite database. The output is an answer file listing the required
calculators and their results. The agent must identify which calculations apply, retrieve their
inputs from records containing distractors, and follow dependencies or conditional branches between
calculations to produce $2$--$10$ required outputs.

\textbf{Long-context retrieval (\cmem, $23$ tasks).} The input is a question and a patient history
stored as files, spanning $70$--$100$ dialogue sessions. The output is a short factual answer,
such as a dose, lab value or diagnosis. The history exceeds a single context window, and the
relevant evidence can lie dozens of sessions before the question. The agent must locate that
evidence in the files and bring it into context to answer correctly.

\textbf{Clinical data workflows (\cplan, $28$ tasks).} Each task provides an instruction and a
workspace of clinical data, and requires structured result files. Tasks include trial eligibility
screening, EHR processing and tumor localization on whole-slide images. The challenge is to
complete a workflow across many documents, records or image regions. Trial screening, for example,
requires finding all eligible trials among hundreds of candidates.

\textbf{Browser interaction (\cweb, $24$ tasks).} The input is a workflow request, such as resolving
an insurance denial, and a browser session across an EMR, a payer portal and a fax portal.
The output is the updated application state. The agent must inspect records, navigate between
applications and carry out the required updates, keeping track of which steps remain unfinished.

\subsection{Scoring}
\label{sec:scoring}

\textbf{Outcome.} Each task receives a score in $[0,1]$. We average scores over tasks within each
environment. \ctool multiplies calculator-selection F1 by the fraction of all required outputs
that are correct. The other three environments use binary outcomes. \cmem matches normalized
answers against accepted aliases. \cplan applies the task's own checks, and \cweb requires every
outcome check on the final application state to pass.

\textbf{Task diagnostics (Proc.).} Alongside outcome, \Cref{tab:main-results} reports
calculator-selection F1 for \ctool, reference-set recall for \cplan and the share of process checks
passed for \cweb.
For \cmem, which has no intermediate environment checks, an LLM judge scores a condensed
interaction trace reconstructed from the proxy logs (\Cref{sec:setup}). Its score averages three dimensions:
appropriate tool use, coherent and evidence-grounded actions, and recovery from tool failures.
\Cref{sec:appendix-grading} gives the scoring rules and trace extraction procedure.

\section{\lgharness}
\label{sec:mhlab}

To study individual harness mechanisms, we build \lgharness (\mbox{MedicalHarness} Lab), a minimal
harness with a shared execution loop and three component switches. Implemented in
LangGraph~\citep{langgraph}, it follows a ReAct-style loop~\citep{yao2023react} that constructs
the model's context, calls the model, executes its tool calls, and updates the state for the next
turn (\Cref{fig:mhlab}).

We focus on three common mechanisms that govern how agents retain information, track unfinished
work, and access tools. Medical tasks often accumulate evidence across records, require dependent
steps, and involve repeated interactions with files or clinical
applications. Context management determines which parts of a growing history remain available to
the model. Planning maintains an explicit record of progress and remaining steps. Tool exposure
determines how the model discovers and invokes available tools. Each mechanism can be varied within
the same loop while preserving the task environment and underlying task tools.

\begin{wrapfigure}[14]{R}{0.43\textwidth}
  \vspace{-14pt}
  \centering
  \captionsetup{skip=3pt}
  \includegraphics[width=\linewidth]{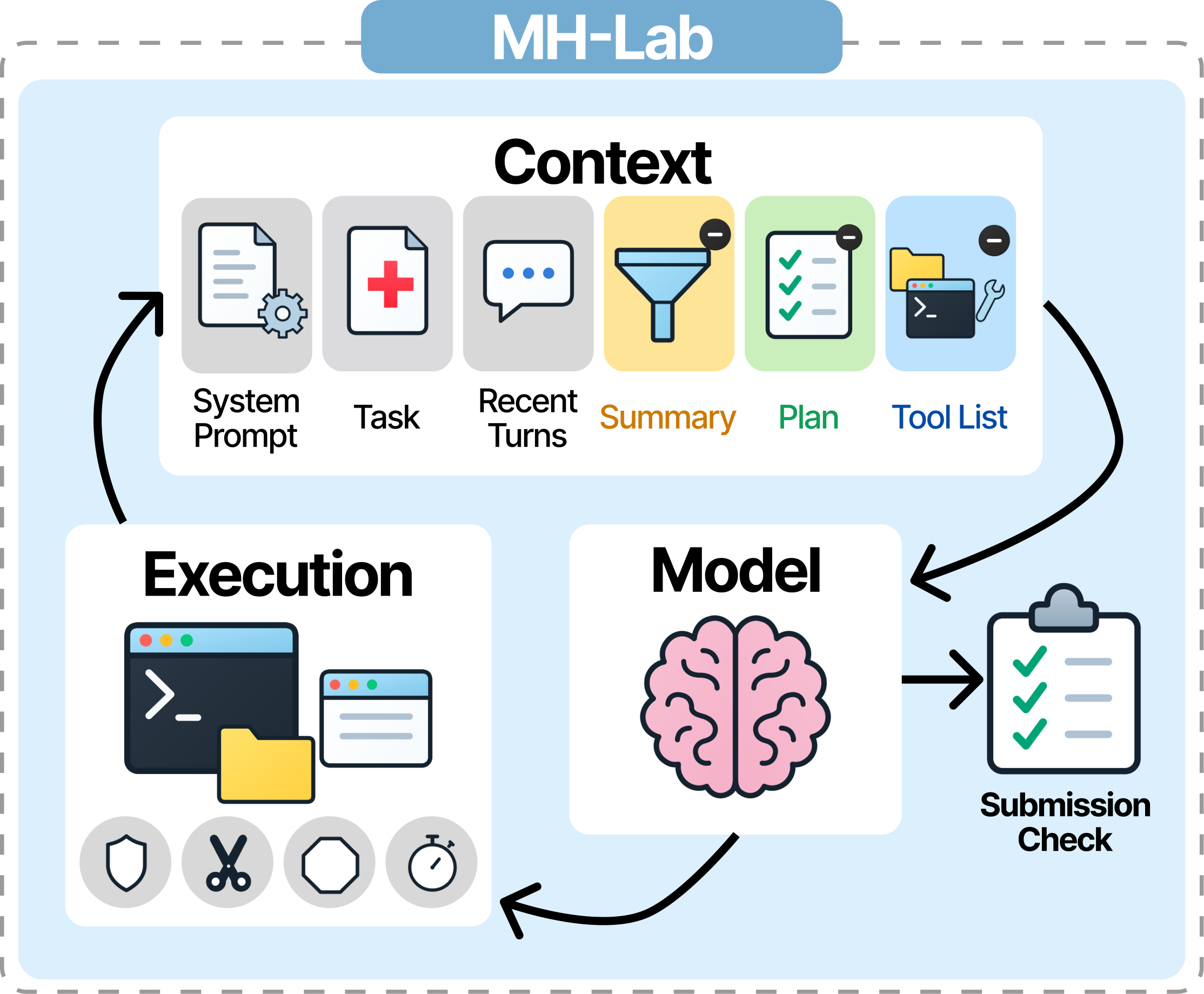}
  \caption{\small\textbf{\lgharness.} A shared execution loop with three component switches.}
  \label{fig:mhlab}
  \vspace{-1\baselineskip}
\end{wrapfigure}
\textbf{\mbox{Context management.}} As the context fills, the same model summarizes older history
for subsequent requests while retaining recent interactions.
Summary calls count toward the episode's budget.
Switching it off retains the accumulated history without summarization.
\textbf{\mbox{Planning.}} An \texttt{update\_plan} tool and plan-related instructions support
tracking progress. The current plan appears on every turn to make completed and unfinished steps
explicit. Switching it off removes the tool,
instructions, reminders and displayed plan.
\textbf{\mbox{Tool exposure.}} Task tools are exposed with their names, descriptions and argument
schemas. Switching off direct exposure keeps the same tools and executors but requires discovery
before calls~through~\texttt{tool\_call}.

For each comparison, we change one mechanism while keeping the model, task instructions, underlying
task tools, stopping rules, and episode budget fixed. We record model inputs, tool calls, plan
updates, summary events, and termination reasons to relate changes in outcomes to changes in
execution. Implementation details are provided in \Cref{sec:appendix-mhloop}.

\section{Experiments}
\label{sec:experiments}

Our experiments address three research questions. We first measure performance variation
across models and harnesses (\hyperref[sec:harness-effect]{RQ1}), then examine how harnesses
differ in task execution (\hyperref[sec:delivery]{RQ2}). Finally, we use \lgharness to test the
effects of individual mechanisms with two models on all four environments
(\hyperref[sec:components]{RQ3}).

\subsection{Experimental setup}
\label{sec:setup}

\textbf{Models and harnesses.} We evaluate five open-weight models spanning $26$B to $120$B parameters:
\mgemma, \mqwen, \mgptoss, \mlightning and \mglm. The main table shows the first three, with the other
two in \Cref{tab:appendix-results}. Each model is served locally by vLLM~\citep{kwon2023vllm} with one
checkpoint and context window, shared by five harnesses through their custom-provider interfaces:
OpenClaw~\citep{openclaw}, Hermes~\citep{hermesagent}, ZeroClaw~\citep{zeroclaw}, Codex~\citep{codexcli}
and Claude Code~\citep{claudecode}. Harnesses keep their shipped settings apart from the changes listed
in \Cref{sec:appendix-harness}.

\textbf{Execution.} An episode is one run of a task by a model--harness pair under one seed, with a
one-hour wall-clock limit, in a sandbox that hides graders, gold answers and other episodes. A proxy on
every model call records requests, responses and token counts, forming the execution traces of
\Cref{sec:delivery}.

\textbf{Measures.} Besides outcome and task diagnostics (\Cref{sec:scoring}), we report the answer rate on
\ctool and \cmem, the fraction of episodes with a scorable answer in the answer file or final reply, and
token use, the input and output tokens of all model calls in millions per episode (MTok).

\textbf{Protocol.} Every model--harness pair runs all $107$ tasks with seeds $42$, $43$ and $44$, for
$8{,}025$ episodes. The mechanism study runs \mqwen and \mgptoss in \lgharness on the same tasks and seeds,
with a $65{,}536$-token input allowance per request. Tables report means and SDs over seeds, and paired
comparisons bootstrap tasks after averaging over seeds, giving $95\%$
intervals~\citep{card2020power,miller2024errorbars,madaan2024variance,heineman2025signal}.

\begin{table*}[t]
  \centering
  \caption{\small\textbf{Outcome and process under a fixed model.} Only the harness varies within a block. Mean $\pm$ SD over seeds. \textbf{Proc.}: calculator selection F1 (\ctool), recall of the reference set on the $17$ \cplan tasks that have one, share of process checks passed (\cweb). $^\dagger$\cmem records only the answer, so its Proc.\ is the LLM judge's score (\cref{sec:appendix-grading}). \textbf{MTok}: million tokens per episode. \emph{Rank}: mean outcome rank within a block. \textbf{Bold} and \underline{underlined}: best and second best Outc.\ and Proc.\ per capability within a block. Rows are ordered by mean rank over all five models, worst first.}
  \label{tab:main-results}
  \setlength{\tabcolsep}{3pt}
  \renewcommand{\arraystretch}{1.12}
  \resizebox{\textwidth}{!}{%
  \begin{tabular}{lI cccI cccI cccI cccI c}
    \Xhline{1.2pt}
    \multirow{2}{*}{\textbf{Harness}}
      & \multicolumn{3}{cI}{\textbf{\ctool}}
      & \multicolumn{3}{cI}{\textbf{\cmem}}
      & \multicolumn{3}{cI}{\textbf{\cplan}}
      & \multicolumn{3}{cI}{\textbf{\cweb}}
      & \multirow{2}{*}{\textbf{Rank}} \\
      & \textbf{Outc.}$\uparrow$  & \textbf{Proc.}$\uparrow$  & \textbf{MTok}$\downarrow$  & \textbf{Outc.}$\uparrow$  & \textbf{Proc.}$\uparrow$$^\dagger$  & \textbf{MTok}$\downarrow$  & \textbf{Outc.}$\uparrow$  & \textbf{Proc.}$\uparrow$  & \textbf{MTok}$\downarrow$  & \textbf{Outc.}$\uparrow$  & \textbf{Proc.}$\uparrow$  & \textbf{MTok}$\downarrow$ & \\
    \Xhline{0.5pt}
    \rowcolor{mhBlock!70}
    \multicolumn{14}{c}{\color{mhNavy}\textbf{\mgemma}} \\
    \baseline{ZeroClaw}      & \meanstd{0.08}{0.01} & \meanstd{0.37}{0.01} & 0.11      & \meanstd{0.49}{0.03} & \meanstd{0.32}{0.01} & 0.48      & \secondresult{\meanstd{0.01}{0.02}} & \meanstd{0.20}{0.02} & 1.79      & \meanstd{0.03}{0.02} & \meanstd{0.20}{0.05} & 6.67      & 3.75 \\
    \baseline{OpenClaw}      & \secondresult{\meanstd{0.10}{0.01}} & \meanstd{0.45}{0.04} & 0.52      & \meanstd{0.57}{0.08} & \meanstd{0.35}{0.05} & 1.04      & \meanstd{0.00}{0.00} & \meanstd{0.20}{0.01} & 4.51      & \meanstd{0.04}{0.04} & \firstresult{\meanstd{0.38}{0.06}} & 11.95      & 3.00 \\
    \baseline{Hermes}      & \firstresult{\meanstd{0.13}{0.04}} & \secondresult{\meanstd{0.46}{0.05}} & 0.11      & \meanstd{0.57}{0.12} & \meanstd{0.23}{0.07} & 0.78      & \meanstd{0.00}{0.00} & \firstresult{\meanstd{0.32}{0.04}} & 1.04      & \meanstd{0.01}{0.02} & \meanstd{0.15}{0.13} & 4.02      & 3.75 \\
    \baseline{Codex}      & \meanstd{0.09}{0.01} & \firstresult{\meanstd{0.47}{0.04}} & 0.06      & \secondresult{\meanstd{0.77}{0.05}} & \firstresult{\meanstd{0.48}{0.07}} & 0.73      & \meanstd{0.01}{0.02} & \secondresult{\meanstd{0.31}{0.04}} & 2.30      & \firstresult{\meanstd{0.11}{0.02}} & \meanstd{0.29}{0.05} & 20.04      & 2.25 \\
    \baseline{Claude Code}      & \meanstd{0.06}{0.02} & \meanstd{0.28}{0.01} & 0.29      & \firstresult{\meanstd{0.78}{0.04}} & \secondresult{\meanstd{0.36}{0.14}} & 1.23      & \firstresult{\meanstd{0.04}{0.04}} & \meanstd{0.29}{0.07} & 4.75      & \secondresult{\meanstd{0.10}{0.10}} & \secondresult{\meanstd{0.30}{0.02}} & 9.66      & 2.25 \\
    \Xhline{0.5pt}
    \rowcolor{mhBlock!70}
    \multicolumn{14}{c}{\color{mhNavy}\textbf{\mqwen}} \\
    \baseline{ZeroClaw}      & \meanstd{0.08}{0.03} & \meanstd{0.25}{0.06} & 0.41      & \meanstd{0.67}{0.07} & \meanstd{0.30}{0.07} & 0.48      & \meanstd{0.02}{0.02} & \meanstd{0.19}{0.05} & 4.03      & \meanstd{0.32}{0.09} & \meanstd{0.45}{0.02} & 6.73      & 5.00 \\
    \baseline{OpenClaw}      & \meanstd{0.10}{0.01} & \meanstd{0.42}{0.03} & 0.47      & \meanstd{0.83}{0.09} & \meanstd{0.44}{0.04} & 0.83      & \firstresult{\meanstd{0.10}{0.02}} & \secondresult{\meanstd{0.47}{0.03}} & 2.84      & \secondresult{\meanstd{0.46}{0.04}} & \meanstd{0.47}{0.04} & 7.93      & 2.50 \\
    \baseline{Hermes}      & \firstresult{\meanstd{0.13}{0.03}} & \meanstd{0.45}{0.05} & 0.24      & \meanstd{0.78}{0.08} & \secondresult{\meanstd{0.55}{0.04}} & 0.41      & \meanstd{0.08}{0.05} & \firstresult{\meanstd{0.50}{0.06}} & 1.69      & \meanstd{0.33}{0.07} & \meanstd{0.42}{0.07} & 7.83      & 3.00 \\
    \baseline{Codex}      & \secondresult{\meanstd{0.12}{0.01}} & \firstresult{\meanstd{0.50}{0.01}} & 0.23      & \firstresult{\meanstd{0.86}{0.03}} & \firstresult{\meanstd{0.65}{0.04}} & 0.62      & \meanstd{0.06}{0.05} & \meanstd{0.45}{0.01} & 2.13      & \meanstd{0.42}{0.04} & \secondresult{\meanstd{0.47}{0.03}} & 5.08      & 2.50 \\
    \baseline{Claude Code}      & \meanstd{0.10}{0.02} & \secondresult{\meanstd{0.46}{0.03}} & 0.25      & \secondresult{\meanstd{0.86}{0.10}} & \meanstd{0.42}{0.03} & 0.62      & \secondresult{\meanstd{0.10}{0.04}} & \meanstd{0.45}{0.03} & 2.68      & \firstresult{\meanstd{0.47}{0.05}} & \firstresult{\meanstd{0.48}{0.01}} & 8.11      & 2.00 \\
    \Xhline{0.5pt}
    \rowcolor{mhBlock!70}
    \multicolumn{14}{c}{\color{mhNavy}\textbf{\mgptoss}} \\
    \baseline{ZeroClaw}      & \meanstd{0.12}{0.02} & \meanstd{0.46}{0.03} & 0.10      & \meanstd{0.36}{0.03} & \meanstd{0.26}{0.07} & 0.29      & \meanstd{0.01}{0.02} & \meanstd{0.11}{0.06} & 0.63      & \meanstd{0.00}{0.00} & \secondresult{\meanstd{0.27}{0.02}} & 1.24      & 4.25 \\
    \baseline{OpenClaw}      & \meanstd{0.12}{0.01} & \meanstd{0.43}{0.02} & 0.24      & \meanstd{0.51}{0.05} & \meanstd{0.32}{0.03} & 0.51      & \meanstd{0.02}{0.02} & \meanstd{0.20}{0.04} & 0.94      & \secondresult{\meanstd{0.03}{0.02}} & \meanstd{0.22}{0.03} & 1.77      & 3.75 \\
    \baseline{Hermes}      & \meanstd{0.12}{0.01} & \secondresult{\meanstd{0.49}{0.03}} & 0.10      & \secondresult{\meanstd{0.68}{0.07}} & \meanstd{0.22}{0.01} & 0.37      & \secondresult{\meanstd{0.04}{0.04}} & \secondresult{\meanstd{0.38}{0.03}} & 0.61      & \meanstd{0.00}{0.00} & \meanstd{0.06}{0.03} & 0.22      & 2.75 \\
    \baseline{Codex}      & \secondresult{\meanstd{0.13}{0.02}} & \meanstd{0.44}{0.06} & 0.22      & \meanstd{0.68}{0.07} & \firstresult{\meanstd{0.38}{0.04}} & 0.51      & \meanstd{0.04}{0.00} & \firstresult{\meanstd{0.42}{0.08}} & 0.63      & \meanstd{0.00}{0.00} & \meanstd{0.00}{0.00} & 0.10      & 3.25 \\
    \baseline{Claude Code}      & \firstresult{\meanstd{0.16}{0.02}} & \firstresult{\meanstd{0.52}{0.02}} & 0.30      & \firstresult{\meanstd{0.80}{0.09}} & \secondresult{\meanstd{0.35}{0.06}} & 0.47      & \firstresult{\meanstd{0.05}{0.04}} & \meanstd{0.28}{0.03} & 1.55      & \firstresult{\meanstd{0.26}{0.09}} & \firstresult{\meanstd{0.33}{0.09}} & 4.09      & 1.00 \\
    \Xhline{1.2pt}
  \end{tabular}}
\end{table*}

\subsection{RQ1: How much does the harness matter?}
\label{sec:harness-effect}

\begin{figure*}[!t]
  \centering
  \includegraphics[width=\textwidth]{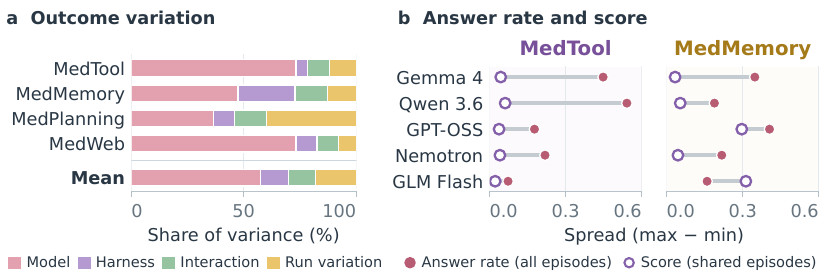}
  \caption{\small\textbf{Outcome variation and answer production.}
  Five models, five harnesses and three seeds.
  \textbf{(a)} Variance shares of per-seed mean outcomes (\cref{eq:anova}). Mean weights the four
  environments equally. \textbf{(b)} Per-model spreads (maximum minus minimum) across harnesses
  in answer rate over all episodes and mean score on matched task--seed pairs answered by every
  harness. Score is the task outcome (\Cref{sec:scoring}). Answers in either the file or
  final reply count.}
  \label{fig:harness-comparison}
\end{figure*}

\Cref{tab:main-results} compares harnesses under the three main models.
\Cref{fig:harness-comparison}a splits the outcome variance of all five models into four shares, for
the model, the harness, their interaction and run-to-run variation across seeds (\Cref{eq:anova}).
We ask how large the harness share is and whether the best harness stays the same across models
and tasks. Two findings emerge.

\noindent\textbf{The model explains the largest share, but the harness matters differently
across tasks.} Averaged over the four environments, the model accounts for $58\%$ of the variance,
the harness for $12\%$, their interaction for $12\%$ and run-to-run variation across seeds for the
remaining $18\%$. The harness and its interaction thus account for about a quarter, but the shares
differ by environment. On \cmem the harness main effect reaches $25\%$, and $40\%$ with the
interaction, while on \cplan, where success is rare, run-to-run variation accounts for
$40\%$, comparable to the model's $36\%$.
\Cref{prop:hvmv} compares the two directions directly: across all five models, HV/MV ranges from
$0.18$ on \ctool to $0.64$ on \cmem, so changing the model moves scores more than changing the
harness everywhere. Among the three main models, HV/MV on \cmem rises to $1.45$, the only case
where the harness main effect exceeds the model's (\Cref{sec:appendix-variance}).

\textbf{No harness is best across models and tasks.} The leader differs by model even within one
environment. On \ctool, Hermes leads three models, Claude Code leads \mgptoss, and Codex leads
\mlightning. Leads are also small relative to uncertainty from the sampled tasks. Only $2$ of the
$20$ model--environment pairs have a $95\%$ interval for the leader's gap over the runner-up that
excludes zero, and both are Claude Code under \mgptoss, on \cmem and \cweb. The clearest leads thus
hold for a single model. \Cref{tab:ranking-stability} gives all rankings and intervals.

\subsection{RQ2: How do harnesses differ in task execution?}
\label{sec:delivery}

We examine whether episodes end with an answer and what makes each environment difficult.

\textbf{Harnesses differ in whether an episode ends with an answer.} \fone~On \ctool and \cmem,
\Cref{fig:harness-comparison}b compares the spread across harnesses in answer rate, over all
episodes, with the spread in score, over the task--seed pairs every harness answered. The
answer-rate spread is the larger in $9$ of the $10$ model--environment pairs. Scores on these shared
questions differ by only $0.02$--$0.06$ in eight pairs, while answer rates differ by up to $0.54$.
The exception is \mglm on \cmem
(\Cref{sec:appendix-execution}). \ftwo~Unsuccessful episodes end in different
ways. Under \mqwen (\Cref{fig:task-diagnostics}a), ZeroClaw reaches the wall-clock limit in $50$ of
$96$ \ctool episodes, while Codex and Claude Code usually produce a scorable output. A single outcome score hides these
differences, so we report answer rates alongside it (\Cref{sec:appendix-proofs} decomposes the variance).

\begin{figure*}[!t]
  \centering
  \includegraphics[width=\textwidth]{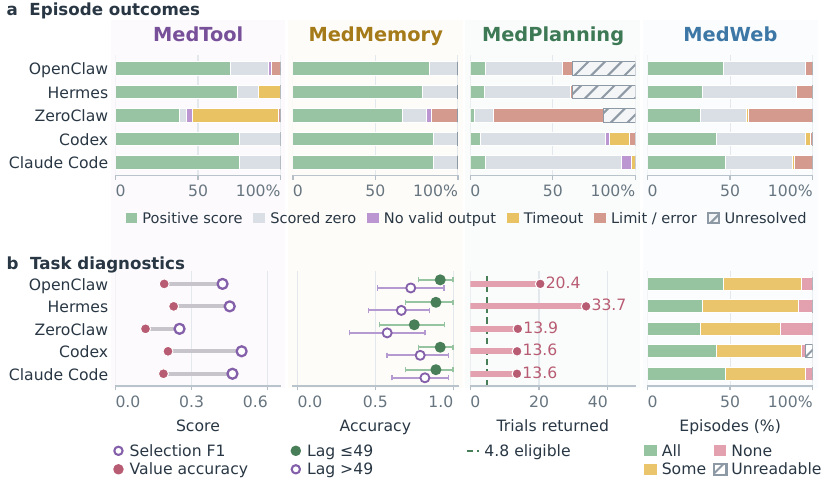}
  \captionsetup{skip=3pt}
  \caption{\small\textbf{Episode outcomes and task diagnostics under \mqwen}, over three seeds.
  \textbf{(a)} Episodes grouped by score, ending and output evidence. Positive score includes
  partial credit. \textbf{(b)} \ctool: selection F1 and value accuracy. \cmem: accuracy by
  evidence lag in sessions, with $95\%$ task-bootstrap intervals. \cplan: mean trials returned
  versus eligible on nine trial tasks, while (a) covers all \cplan tasks. \cweb: episodes passing
  all, some or no outcome checks, with unreadable states separate.}
  \label{fig:task-diagnostics}
  \vspace{-8pt}
\end{figure*}

\textbf{Each environment exposes a different difficulty.} \Cref{fig:task-diagnostics}b shows the task
diagnostics. \fone~On \ctool, calculator-selection F1 is about twice value accuracy
across harnesses, so choosing the right calculators does not ensure correct values.
\ftwo~On \cmem, questions with more distant evidence have lower accuracy under every harness. \fthree~On the \cplan trial tasks, agents return $13.6$ to
$33.7$ trials on average against $4.8$ eligible, three to seven times too many. \ffour~On \cweb, fewer than half of episodes pass every
outcome check under any harness, and about half pass only some. These diagnostics separate selecting
tools, retrieving evidence, filtering candidates and completing workflows, which one score
would compress.

\textbf{Mechanisms inside the product harnesses.} At the native windows, summarization fires in only
$2$--$5\%$ of episodes, which caps its share of these gaps at that fraction. On \ctool and \cmem, the gaps
come mainly from answer delivery, and switching the matching mechanism inside a product harness moves
outcomes accordingly: ZeroClaw's submission check raises \mqwen's \cmem score from $0.42$ to $0.51$, and
turning off Hermes's tool search lifts \mqwen from $0.17$ to $0.50$ on $12$ \cweb tasks
(\Cref{sec:appendix-modules}).

\subsection{RQ3: Which mechanisms matter, and when?}
\label{sec:components}

\Cref{tab:final-switches} compares the full \lgharness with variants that disable one mechanism
at a time, keeping task settings and other execution limits fixed (\Cref{sec:mhlab}). Its input
allowance is below the native windows, where summarization is rare (\Cref{sec:appendix-modules}), so
that context management fires.

\begin{table*}[t]
  \centering
  \caption{\small\textbf{Removing one mechanism from \lgharness.} $65{,}536$-token input budget, seeds 42/43/44. \emph{full}: outcome and million tokens per episode, mean $\pm$ SD over seeds. Other rows: paired change in outcome ($\pm$ SD of the per-seed change) and token ratio to \emph{full}. \Cref{tab:final-intervals} gives the intervals.}
  \label{tab:final-switches}
  \setlength{\tabcolsep}{4pt}\renewcommand{\arraystretch}{1.1}
  \resizebox{\textwidth}{!}{%
  \begin{tabular}{lI cc I cc I cc I cc}
    \Xhline{1.2pt}
    \multirow{2}{*}{\textbf{Condition}} & \multicolumn{2}{cI}{\textbf{\ctool}} & \multicolumn{2}{cI}{\textbf{\cmem}} & \multicolumn{2}{cI}{\textbf{\cplan}} & \multicolumn{2}{c}{\textbf{\cweb}} \\
    & \textbf{Outc.}$\uparrow$ & \textbf{MTok}$\downarrow$ & \textbf{Outc.}$\uparrow$ & \textbf{MTok}$\downarrow$ & \textbf{Outc.}$\uparrow$ & \textbf{MTok}$\downarrow$ & \textbf{Outc.}$\uparrow$ & \textbf{MTok}$\downarrow$ \\
    \Xhline{0.5pt}
    \rowcolor{mhBlock!70}
    \multicolumn{9}{c}{\color{mhNavy}\textbf{\mqwen}} \\
    \emph{full} & \meanstd{0.10}{0.02} & \meanstd{0.23}{0.01} & \meanstd{0.88}{0.03} & \meanstd{0.43}{0.12} & \meanstd{0.12}{0.05} & \meanstd{1.62}{0.26} & \meanstd{0.44}{0.02} & \meanstd{5.96}{0.38} \\
    $-$ context management & \meanstd{$+$0.01}{0.02} & \meanstd{$\times$1.1}{0.0} & \meanstd{$-$0.12}{0.03} & \meanstd{$\times$0.5}{0.2} & \meanstd{$-$0.04}{0.06} & \meanstd{$\times$0.4}{0.1} & \meanstd{$-$0.39}{0.09} & \meanstd{$\times$0.3}{0.0} \\
    $-$ planning & \meanstd{$-$0.02}{0.04} & \meanstd{$\times$1.0}{0.1} & \meanstd{$-$0.04}{0.04} & \meanstd{$\times$0.8}{0.4} & \meanstd{$+$0.01}{0.08} & \meanstd{$\times$0.9}{0.2} & \meanstd{0.00}{0.07} & \meanstd{$\times$0.7}{0.2} \\
    $-$ tool exposure & \meanstd{0.00}{0.02} & \meanstd{$\times$3.0}{1.6} & \meanstd{$-$0.14}{0.05} & \meanstd{$\times$2.4}{0.7} & \meanstd{$+$0.01}{0.05} & \meanstd{$\times$2.6}{0.5} & \meanstd{0.00}{0.07} & \meanstd{$\times$0.8}{0.1} \\
    \Xhline{0.5pt}
    \rowcolor{mhBlock!70}
    \multicolumn{9}{c}{\color{mhNavy}\textbf{\mgptoss}} \\
    \emph{full} & \meanstd{0.14}{0.03} & \meanstd{0.28}{0.07} & \meanstd{0.72}{0.03} & \meanstd{0.66}{0.13} & \meanstd{0.06}{0.02} & \meanstd{0.99}{0.38} & \meanstd{0.18}{0.08} & \meanstd{7.70}{1.58} \\
    $-$ context management & \meanstd{$-$0.01}{0.04} & \meanstd{$\times$1.0}{0.3} & \meanstd{$+$0.06}{0.05} & \meanstd{$\times$0.8}{0.1} & \meanstd{$+$0.02}{0.02} & \meanstd{$\times$0.6}{0.3} & \meanstd{$-$0.17}{0.14} & \meanstd{$\times$0.4}{0.1} \\
    $-$ planning & \meanstd{$-$0.02}{0.04} & \meanstd{$\times$0.4}{0.1} & \meanstd{$-$0.01}{0.03} & \meanstd{$\times$0.7}{0.2} & \meanstd{$+$0.02}{0.02} & \meanstd{$\times$0.9}{0.4} & \meanstd{$-$0.18}{0.09} & \meanstd{$\times$0.2}{0.1} \\
    $-$ tool exposure & \meanstd{$-$0.04}{0.06} & \meanstd{$\times$3.2}{1.8} & \meanstd{$-$0.06}{0.07} & \meanstd{$\times$1.3}{0.4} & \meanstd{$-$0.04}{0.00} & \meanstd{$\times$1.9}{0.7} & \meanstd{$-$0.15}{0.14} & \meanstd{$\times$1.0}{0.0} \\
    \Xhline{1.2pt}
  \end{tabular}}
\end{table*}

\textbf{Context management keeps long trajectories running.} Without summaries, accumulated
history overflows the input budget in $67$ of $72$ \cweb episodes for \mqwen and $48$ of $65$
for \mgptoss (\Cref{tab:final-mechanisms}). Outcomes fall by $0.39$ and $0.17$, respectively. Lower token use reflects early
termination. On \cplan,
binary success stays rare, but reference-set recall falls by $0.29$ for \mqwen and $0.16$ for \mgptoss
(\Cref{tab:final-intervals}), so summaries preserve partial progress.

\textbf{Planning helps some models persist.} Without the plan tool and displayed plan,
\mgptoss's median turns on the $60$ paired \cweb episodes fall from $176.5$ to $63$, and outcome and
process scores drop by $0.18$ and $0.19$. \mqwen shows no clear outcome change on \cweb. On \ctool, however, its episodes
without the required output rise from $24$ to $47$ of $96$ and timeouts double. Similar scores
can conceal changes in whether the agent finishes.

\textbf{Tool interfaces must fit the model's calling behavior.} With tools behind search and
\texttt{tool\_call}, models discover a tool but invoke it by name, causing rejected calls.
This occurs in $95$ of $96$ \ctool episodes for \mqwen, which uses three times as many tokens
at the same outcome. \mgptoss loses outcome on all four environments, most on \cweb ($-0.15$). These failures
arise when discovery must be followed by \texttt{tool\_call}, so a discovery interface should make
the found tool directly callable, as Claude Code does.

\section{Conclusion}
\label{sec:conclusion}

We introduced \benchsuite and used it to measure how outcomes vary with the model and harness
under controlled conditions. The model explains the largest share, but the harness and its interaction
with the model account for about a quarter of the variance, and no single harness is best across models
and tasks. On clinical calculation and long-context retrieval, harnesses usually
differ more in answer rate than in scores on questions answered by every harness. A mechanism changes outcomes only when it fires and when its interface fits the model. Inside \lgharness, context management
decides whether long traces survive and planning sustains longer browser workflows for \mgptoss. Hiding tools behind a search
interface can cost tokens or accuracy when models call the tools by name. Medical agent
results should therefore report the model--harness pair and, when an answer is required, its
answer rate alongside the score. \Cref{sec:appendix-design} turns these findings into design directions
for medical agent harnesses.

\ifdefined\MedicalHarnessArxiv
\else
  \section*{AI Use Statement}
We used Large Language Models (ChatGPT/Claude/Gemini) for grammatical correction and editing of
this manuscript, and as coding assistants for parts of the evaluation infrastructure, the analysis
scripts and \lgharness. Separately, language models served as instruments of the study:
\texttt{gpt-4o-mini} scored the process rubric (\Cref{sec:appendix-grading}), and non-target models
reviewed candidate tasks during verification (\Cref{sec:appendix-bench}). The authors reviewed all code
and verified every reported number against the archived episode records. All technical
contributions, analyses and scientific conclusions are the authors' own responsibility.

  \section*{Ethics Statement}
\methodname is a research study of LLM agents on medical tasks. The agents evaluated here are
research systems, and their scores do not indicate fitness for clinical use, which would require
clinical validation and expert review. All tasks derive from public benchmarks built on synthetic or
de-identified data: MedMCP-Calc (MIT), HealthAgentBench (MIT, only tasks that need no credentialed
data), HealthAdminBench (Apache-2.0, synthetic portal records) and MedMemoryBench (synthetic patient
personas seeded from de-identified cases). No task uses credentialed clinical data. Agents ran on
locally served models, and only the process judge received condensed episode traces through an
external API, all from synthetic or public tasks. Because the MedMemoryBench data license is stated
inconsistently upstream, we release item identifiers, hashes and the construction code rather than
the chart text.

  \section*{Reproducibility Statement}
Every model is an open-weight checkpoint served locally with a fixed context window (\Cref{sec:setup}),
and every configuration was run with seeds $42$, $43$ and $44$. \Cref{sec:appendix-bench} describes how
each environment was built and verified, \Cref{sec:appendix-harness} lists the harness versions and the
defaults we changed, \Cref{sec:appendix-grading} gives the grading rules, and \Cref{sec:appendix-mhloop} specifies
\lgharness. Code, harness configurations and task manifests are available at
\repositorylink, with MedMemoryBench items released as
identifiers and hashes together with their construction code.

\fi

\bibliography{iclr2027_conference}
\bibliographystyle{iclr2027_conference}

\clearpage
\appendix
\section{Design implications for medical agent harnesses}
\label{sec:appendix-design}

Our execution diagnostics and mechanism experiments (\Cref{sec:delivery,sec:components}) suggest
three design directions for harnesses that carry out medical tasks.

\textbf{Preserve clinical evidence.} On \cmem, accuracy falls as supporting evidence lies further
back in the record. The context-management experiments also show that summaries can keep long
trajectories running when history exceeds the input budget. These findings motivate context
management that controls length while keeping relevant evidence accessible. A task-specific evidence
record could retain key facts with their dates and source locations, allowing the agent to revisit
the original record when needed. For example, after compression, a medication entry should still
indicate when it was recorded and whether treatment was ongoing or had ended.

\textbf{Ground plans in clinical criteria.} On trial-screening tasks, agents frequently return
ineligible candidates. The planning experiments also show that maintaining a plan helps some
models sustain longer workflows, but does not consistently improve outcomes. A useful direction is
to make each plan step resolve a specific clinical condition. For trial eligibility, the plan could
track each candidate's inclusion and exclusion criteria, the supporting evidence and whether each
condition is met, not met or unknown. Unresolved conditions would guide further retrieval before
the agent decides whether to include the trial in its final result.

\textbf{Verify task completion.} On \ctool, selecting the right calculators does not ensure correct
values, and on \cweb, many episodes satisfy only some outcome checks. These observations motivate
checks of observable task requirements after key actions and before stopping. A simple version
already helps when the model has the answer but has not filed it: enabling ZeroClaw's submission
check, which adds a reminder when the answer file is missing, raises \mqwen's \cmem score from $0.42$
to $0.51$ (\Cref{sec:appendix-modules}). Calculator workflows
could check that every requested quantity has a value and that its inputs match the patient record.
Browser workflows could reread the relevant portal state after an action, such as updating an
authorization request, to confirm that the intended change took effect. Such checks would identify
incomplete steps and give the agent a concrete target for a retry.

\section{Benchmark construction}
\label{sec:appendix-bench}

We build each environment from a public benchmark, selecting the tasks that exercise its capability
and verifying every task we keep (\Cref{tab:appendix-bench}). Of the $2{,}293$ upstream tasks, $107$
are kept. Automatic checks and, on \ctool and \cmem, models outside the study narrow the candidates.
A human expert then reviews every kept task, checking that its gold answer is correct, that neither
the prompt nor the recent record gives the answer away, that the task is unambiguous and that it is
clinically sound. \Cref{tab:task-composition} expands the descriptive groups in
\Cref{fig:task-composition}.

\begin{table}[h]
  \centering
  \caption{\small\textbf{How each environment was built.} Each selection step lists the tasks it leaves.
  A human expert then reviews every kept task.}
  \label{tab:appendix-bench}
  \footnotesize
  \setlength{\tabcolsep}{3pt}
  \begin{tabular}{l>{\raggedright\arraybackslash}p{2.5cm}r>{\raggedright\arraybackslash}p{4.5cm}r>{\raggedright\arraybackslash}p{2.45cm}}
    \Xhline{1.2pt}
    \textbf{Env.} & \textbf{Source (license)} & \textbf{Pool} & \textbf{Selection (tasks left)} & \textbf{Kept} & \textbf{Verification} \\
    \Xhline{0.5pt}
    \ctool & MedMCP-Calc (MIT) & $118$ & difficulty screen ($107$) $\to$ reliability contract and independent model audit ($32$) & $32$ & scorer accepts the gold and rejects a perturbed answer \\
    \cmem & MedMemoryBench (Apache-2.0) & $1{,}986$ & exact-match answers ($349$) $\to$ long history ($241$) $\to$ answer not restated in recent sessions ($214$) $\to$ entity and date questions whose gold the record supports ($51$) $\to$ two blind model reviews ($14$, plus $9$ from an earlier round) & $23$ & reviewers lock answers before the gold is compared \\
    \cplan & HealthAgentBench (MIT) & $54$ & runs without credentialed data ($28$) & $28$ & upstream task tests \\
    \cweb & HealthAdminBench (Apache-2.0) & $135$ & $2$--$8$ outcome checks and a process check, fixed quota per family ($24$) & $24$ & upstream deterministic checks \\
    \Xhline{1.2pt}
  \end{tabular}
\end{table}

\textbf{\ctool.} Each task gives a clinical request and a patient identifier. The agent reads a
per-task SQLite record through read-only SQL and Python, and the needed values sit among
distractors. The difficulty screen runs GPT-5.4-mini, which we do not evaluate. Tasks it solves from
the given inputs or with plain SQL and Python are too easy, and tasks on which it makes no progress at
all point to knowledge the task does not supply, so the screen removes both. Each remaining task must
then pass a contract: sufficient inputs, a complete and consistent gold answer, a clinically coherent
case, a real dependency or branch between calculators, verified formulas, a scorer that accepts the
gold and rejects a perturbed answer, and an accept verdict from an independent model audit.

\textbf{\cmem.} Each task asks one question about a history of $70$--$100$ consultation sessions, with
the evidence at least $10$ sessions before the question. Upstream questions fail in two main ways.
Some answers are restated in the recent sessions, so the question needs no long-range retrieval.
Others have a gold answer that the visible record does not support, most often a date kept only in
metadata. Automatic checks remove clear cases of both, and blind review catches the rest. Two
reviewers from models outside the study read each remaining question with its history, one after the
other. Each judges whether the question is valid and locks an answer before the gold is revealed, and
a question is kept only when both judge it valid and their answers agree. Comparing the agreed answers
with the gold after the lock then checks each gold answer. Where the reviewers give the same value in
the record's own words, we add that wording as an accepted answer. Answering without the history, or
searching only the recent sessions, scores $0/23$, and a search over the whole chart scores $4/23$.

\textbf{\cplan.} The $28$ tasks are all HealthAgentBench tasks that run without credentialed data, so
the environment can be rerun in full. They cover trial eligibility, EHR data quality, EHR-to-MEDS
conversion and tumor-region selection on whole-slide images (\Cref{tab:task-composition}), and
upstream marks $23$ of them hard and $5$ as controls.

\textbf{\cweb.} We keep tasks whose outcome rests on $2$--$8$ final-state checks and that include at
least one process check, which the Proc.\ diagnostic needs. The check range leaves out the one family
whose outcome joins $9$--$17$ checks. Within each family we take the first tasks by index.

\begin{table}[t]
  \centering
  \caption{\small\textbf{Composition of the $107$ selected task instances.}
  The $13$ display groups summarize source clinical areas and care settings (\ctool),
  requested information (\cmem), and task objectives (\cplan, \cweb).}
  \label{tab:task-composition}
  \footnotesize
  \setlength{\tabcolsep}{3pt}
  \renewcommand{\arraystretch}{1.15}
  \begin{tabular}{@{}>{\raggedright\arraybackslash}p{0.15\linewidth}
    >{\raggedright\arraybackslash}p{0.245\linewidth}
    >{\centering\arraybackslash}p{0.035\linewidth}
    >{\raggedright\arraybackslash}p{\dimexpr0.57\linewidth-6\tabcolsep\relax}@{}}
    \Xhline{1.2pt}
    \textbf{Environment} & \textbf{Display group} & $\boldsymbol{n}$ & \textbf{Detailed content} \\
    \Xhline{0.5pt}
    \ctool\newline ($32$ tasks)
      & Internal Medicine & $14$
      & Gastroenterology/hepatology ($5$), infectious diseases ($3$), respiratory medicine ($2$),
        hematology/oncology ($2$), cardiovascular medicine ($1$), and nephrology/electrolytes ($1$). \\
      & Acute \& Perioperative Care & $11$
      & Emergency medicine/trauma ($5$), intensive care ($2$), and anesthesia/perioperative care ($4$). \\
      & Neurology \& Mental Health & $5$
      & Neurosurgery/neurology ($3$) and mental/psychological health ($2$). \\
      & Nursing \& Pediatrics & $2$
      & Inpatient nursing-risk assessment ($1$) and pediatric-asthma assessment ($1$). \\
    \Xhline{0.5pt}
    \cmem\newline ($23$ tasks)
      & Test Results & $11$
      & Retrieve recorded laboratory, imaging, or measured examination findings, including changes in measured indices. \\
      & Medication History & $7$
      & Retrieve drug identities or classes, doses, dose changes, discontinuation, or the identity of an allergenic medication. \\
      & Diagnoses \& Patient Reports & $5$
      & Recorded diagnoses ($3$), the date of a reported numbness event ($1$), and a reported perceptual change ($1$). \\
    \Xhline{0.5pt}
    \cplan\newline ($28$ tasks)
      & Trial Eligibility & $9$
      & Identify all eligible candidate trials from a patient note and trial inclusion/exclusion criteria. \\
      & Record Processing & $9$
      & Flag erroneous clinical-record rows ($8$), or configure and run MIMIC-IV-to-MEDS data-format conversion ($1$). \\
      & Tumor Localization & $10$
      & Select tumor-containing grid tiles on ten whole-slide pathology images and submit their coordinates. \\
    \Xhline{0.5pt}
    \cweb\newline ($24$ tasks)
      & Denial Resolution & $10$
      & Submit an appeal ($7$), escalate to a supervisor ($2$), or record a corrected-claim disposition ($1$), then document the resolution. \\
      & Referral Review & $10$
      & Verify and clear a referral ($4$), or identify a problem, withhold authorization, and keep the referral open ($6$). \\
      & Equipment Order Checks & $4$
      & Hold an equipment order for an expired evaluation, missing prescription, inactive insurance, or unsigned prescription ($1$ each), then withhold the fax and document the issue. \\
    \Xhline{1.2pt}
  \end{tabular}
  \vspace{3pt}
  \begin{minipage}{\linewidth}
    \footnotesize
    Counts refer to task instances, not unique patients or workflows.
    The eight record-quality tasks include four clue/no-clue pairs.
  \end{minipage}
\end{table}

\textbf{Examples and difficulty.} \Cref{tab:appendix-examples} gives one task per environment, paraphrased
from its prompt, and the best mean outcome any of the $25$ model--harness pairs reaches.

\begin{table}[h]
  \centering
  \caption{\small\textbf{One example per environment and the best outcome over all model--harness pairs.}}
  \label{tab:appendix-examples}
  \footnotesize
  \setlength{\tabcolsep}{4pt}
  \begin{tabular}{l>{\raggedright\arraybackslash}p{9.3cm}c}
    \Xhline{1.2pt}
    \textbf{Env.} & \textbf{Example task (paraphrased)} & \textbf{Best} \\
    \Xhline{0.5pt}
    \ctool & A patient with ulcerative colitis is admitted with frequent bloody stools and systemic symptoms. Decide which severity and risk calculators the question calls for, pull their inputs from the patient's record, compute them, and write every value to the answer file. & $0.16$ \\
    \cmem & Given $80$ consultation sessions in order, answer one question about a dose, lab value or diagnosis that was recorded once, dozens of sessions before the question. & $0.86$ \\
    \cplan & Given one admission note and a directory of trial documents, list every trial whose inclusion criteria the patient meets and none of whose exclusion criteria apply. Other tasks select the tumor region on a whole-slide image. & $0.10$ \\
    \cweb & For a no-authorization claim denial, review the denial, the remittance image and the patient inquiry in the EMR, check eligibility on the payer portal, carry out the right resolution there, and document it in a triage note. & $0.47$ \\
    \Xhline{1.2pt}
  \end{tabular}
\end{table}

\FloatBarrier

\section{Scoring and task diagnostics}
\label{sec:appendix-grading}

Every outcome comes from a deterministic oracle (\Cref{sec:scoring}). For each environment,
\Cref{tab:appendix-scoring} gives the outcome rule, where the answer is read and the task diagnostic
reported as Proc.

\begin{table}[h]
  \centering
  \caption{\small\textbf{Scoring rules.}}
  \label{tab:appendix-scoring}
  \footnotesize
  \setlength{\tabcolsep}{3pt}
  \begin{tabular}{l>{\raggedright\arraybackslash}p{5.5cm}>{\raggedright\arraybackslash}p{2.5cm}>{\raggedright\arraybackslash}p{3.2cm}}
    \Xhline{1.2pt}
    \textbf{Env.} & \textbf{Outcome} & \textbf{Answer read from} & \textbf{Diagnostic (Proc.)} \\
    \Xhline{0.5pt}
    \ctool & calculator-selection F1 $\times$ share of required outputs correct, with omitted outputs counted wrong. Numbers match within $\max(0.1, 0.02|v|)$ of the reference $v$, and categories after normalization & answer file, else final reply & calculator-selection F1 \\
    \cmem & $1$ if the normalized answer matches an accepted alias & answer file, else final reply & LLM judge on the condensed trace \\
    \cplan & $1$ if the task's own tests pass & output files & reference-set recall ($17$~tasks) \\
    \cweb & $1$ if every outcome check on the final state passes & application state & share of process checks passed \\
    \Xhline{1.2pt}
  \end{tabular}
\end{table}

\textbf{Answer recovery.} When the answer file is missing or empty, we parse the final reply with the
task's own scorer, which needs a parseable calculator submission on \ctool or an explicit answer line
on \cmem. A recovered answer counts as \emph{stated}, and an episode with neither a file answer nor a
recoverable reply counts as \emph{absent}. Scoring the file alone would change the harness rankings
(\Cref{sec:appendix-execution}).

\textbf{LLM judge.} The \texttt{gpt-4o-mini} judge scores the process on \cmem, which has no
intermediate checks. It reads a condensed trace rebuilt from the proxy logs: repeated history is
removed, each round's new messages and tool calls are kept, tool calls are read from both the OpenAI
and Anthropic request formats, and long message bodies are shortened. Its score averages three
ratings in $[0,1]$: appropriate tool use, coherent and evidence-grounded actions, and recovery from
tool failures.

\FloatBarrier
\section{Harness configuration and controls}
\label{sec:appendix-harness}

All five harnesses reach the same locally served models through one proxy, which records requests
and responses. \Cref{tab:appendix-harness} lists their versions and the few defaults we changed so
that each harness can complete the tasks.

\begin{table}[h]
  \centering
  \caption{\small\textbf{Harness versions and the defaults we changed.} Versions are those installed at the end of
  the study, except Claude Code, which updates itself and ran one version within each model block.}
  \label{tab:appendix-harness}
  \footnotesize
  \setlength{\tabcolsep}{3pt}
  \begin{tabular}{ll>{\raggedright\arraybackslash}p{8.6cm}}
    \Xhline{1.2pt}
    \textbf{Harness} & \textbf{Version} & \textbf{Default changed, and why} \\
    \Xhline{0.5pt}
    OpenClaw & 2026.9.2 & output cap $16$K $\to$ $65{,}536$ (truncated tool calls on $15$ of $32$ \ctool tasks), clock $900$\,s $\to$ $3{,}600$\,s \\
    Hermes & 0.21.0 & none \\
    ZeroClaw & 0.8.5 & tool-iteration limit $\approx 11$ $\to$ $200$ (past it, the tool list is dropped) \\
    Codex CLI & 0.156.1 & API-key variable supplied, MCP tools auto-approved (\cweb) \\
    Claude Code & 2.1.267--2.1.270 & none (the proxy translates its Anthropic-format traffic) \\
    \Xhline{1.2pt}
  \end{tabular}
\end{table}

\textbf{Serving settings.} vLLM serves \mgemma, \mqwen and \mlightning with a $262{,}144$-token context
window, and \mgptoss and \mglm with $131{,}072$ tokens. On supported request paths the proxy fills in
sampling fields a harness leaves out, with temperature $0.8$ for \mqwen and $1.0$ for the other models
and top-$p$ $0.95$, and it keeps every value a harness sends. In the four model blocks after \mqwen,
the proxy also caps output at $65{,}536$ tokens for \mgemma and \mlightning and $16{,}384$ for \mgptoss
and \mglm, keeping any lower harness limit, and clients time out after $1{,}800$\,s (a stream-idle
timeout for Codex) under a $3{,}600$\,s proxy timeout. In the \mqwen block, output limits come from
the harnesses alone.

\textbf{Isolation and accounting.} Each episode runs once per seed in its own namespace sandbox, with a
private temporary directory, a memory cap and a one-hour deadline, and it cannot see graders, gold
answers or other episodes. Each isolated \cweb episode starts
from a fresh browser context, and no episode is rerun after its score is seen. Before a harness entered
the study, we traced the network connections of one episode to confirm that all model traffic reached
the local proxy. This check excluded one candidate harness, Moltis, which answered with an off-machine
model through credentials it found on the host.

\FloatBarrier
\section{Details of the model--harness decomposition}
\label{sec:appendix-variance}

Within one environment, $\bar y_{mh\sigma}$ is the mean outcome over tasks for model $m$, harness $h$
and seed $\sigma$. A dot denotes an average over that index. The terms in \Cref{eq:anova} are
\begin{align}
\mu &= \bar y_{\cdot\cdot\cdot}, \nonumber\\
\alpha_m &= \bar y_{m\cdot\cdot} - \mu, \qquad
\beta_h = \bar y_{\cdot h\cdot} - \mu, \nonumber\\
\gamma_{mh} &= \bar y_{mh\cdot} - \bar y_{m\cdot\cdot} - \bar y_{\cdot h\cdot} + \mu, \nonumber\\
\varepsilon_{mh\sigma} &= \bar y_{mh\sigma} - \bar y_{mh\cdot}.
\label{eq:anova-terms}
\end{align}
These definitions give $\sum_m \alpha_m = \sum_h \beta_h = 0$,
$\sum_m \gamma_{mh} = 0$ for each $h$, $\sum_h \gamma_{mh} = 0$ for each $m$, and
$\sum_\sigma \varepsilon_{mh\sigma} = 0$ for each pair $(m,h)$.

Writing $M = |\mathcal{M}|$, $H = |\mathcal{H}|$ and $S = |\mathcal{S}|$, the sums of squares are
\begin{align}
\mathrm{SS}_M &= HS\sum_m \alpha_m^2,
& \mathrm{SS}_H &= MS\sum_h \beta_h^2, \nonumber\\
\mathrm{SS}_{MH} &= S\sum_{m,h} \gamma_{mh}^2,
& \mathrm{SS}_E &= \sum_{m,h,\sigma} \varepsilon_{mh\sigma}^2,
\label{eq:ss-components}\\
\mathrm{SS}_T &= \sum_{m,h,\sigma}(\bar y_{mh\sigma} - \mu)^2.
\label{eq:ss-total}
\end{align}
Expanding $\bar y_{mh\sigma} - \mu = \alpha_m + \beta_h + \gamma_{mh} + \varepsilon_{mh\sigma}$,
all cross terms vanish by the zero-sum identities above, for example
$\sum_{m,h,\sigma}\alpha_m\beta_h = S\,(\sum_m\alpha_m)(\sum_h\beta_h) = 0$. This proves \Cref{eq:ss}.
For $\mathrm{SS}_T > 0$, dividing by $\mathrm{SS}_T$ gives the four shares.
We compute these shares separately within each environment and report their arithmetic mean
across the four environments.

\begin{proof}[Proof of \Cref{prop:hvmv}]
Using population variances, a cell at a fixed model deviates from the model mean by
$\bar y_{mh\cdot} - \bar y_{m\cdot\cdot} = \beta_h + \gamma_{mh}$. Hence
\begin{align*}
\overline{\mathrm{HV}}
&= \frac{1}{M}\sum_m \frac{1}{H}\sum_h (\beta_h + \gamma_{mh})^2\\
&= \frac{1}{H}\sum_h \beta_h^2 + \frac{1}{MH}\sum_{m,h}\gamma_{mh}^2
 = \frac{\mathrm{SS}_H + \mathrm{SS}_{MH}}{N}.
\end{align*}
The cross term is zero because $\sum_m\gamma_{mh}=0$ for every $h$.
Exchanging models and harnesses gives
$\overline{\mathrm{MV}} = (\mathrm{SS}_M + \mathrm{SS}_{MH})/N$.
Subtracting the two expressions proves the comparison in \Cref{prop:hvmv}.
\end{proof}

\textbf{Values by environment.} \Cref{tab:appendix-variance} gives the shares behind
\Cref{fig:harness-comparison}a and the HV/MV ratio of \Cref{prop:hvmv}. With all five models the ratio
stays below one in every environment, so changing the model moves scores more than changing the
harness. Among the three main models, \cmem is the only environment where it exceeds one. For
comparison, \citet{zhang2026disclosingharness} report $7.8$ on a SWE-bench subset with three frontier
models and three configurations of one harness, so the ratio depends on which models, harnesses and
tasks are compared.

\begin{table}[h]
  \centering
  \caption{\small\textbf{Variance shares and HV/MV by environment.} Shares of the outcome variance
  over all five models, and the HV/MV ratio over all five and over the three main models.}
  \label{tab:appendix-variance}
  \footnotesize
  \setlength{\tabcolsep}{6pt}
  \begin{tabular}{lccccc}
    \Xhline{1.2pt}
    & \ctool & \cmem & \cplan & \cweb & \textbf{Mean} \\
    \Xhline{0.5pt}
    Model & $73.2\%$ & $47.4\%$ & $36.4\%$ & $73.2\%$ & $57.5\%$ \\
    Harness & $5.3\%$ & $25.4\%$ & $9.6\%$ & $9.5\%$ & $12.4\%$ \\
    Interaction & $9.5\%$ & $14.4\%$ & $14.1\%$ & $9.4\%$ & $11.9\%$ \\
    Run-to-run & $12.0\%$ & $12.8\%$ & $39.8\%$ & $8.0\%$ & $18.2\%$ \\
    \Xhline{0.5pt}
    HV/MV, five models & $0.18$ & $0.64$ & $0.47$ & $0.23$ & \\
    HV/MV, three main models & $0.82$ & $1.45$ & $0.45$ & $0.19$ & \\
    \Xhline{1.2pt}
  \end{tabular}
\end{table}

\FloatBarrier
\section{Additional results}
\label{sec:appendix-models}

\subsection{Performance, rankings and token use}
\label{sec:appendix-performance}

\textbf{Two further models.} \Cref{tab:appendix-results} repeats the main comparison for \mlightning
and \mglm. \mlightning scores near zero on \cplan and \cweb under every harness. \mglm is the only
model under which ZeroClaw is not last overall, and its \cplan outcomes stay near zero: in $92$ of its
$150$ pathology episodes the agent computed a tissue statistic with a threshold that matched almost no
pixel of the H\&E stain and reported the slide as tumor-free.

\begin{table*}[t]
  \centering
  \caption{\small\textbf{Outcome and process under a fixed model.} The two further models. Only the harness varies within a block. Mean $\pm$ SD over seeds. \textbf{Proc.}: calculator selection F1 (\ctool), recall of the reference set on the $17$ \cplan tasks that have one, share of process checks passed (\cweb). $^\dagger$\cmem records only the answer, so its Proc.\ is the LLM judge's score (\cref{sec:appendix-grading}). \textbf{MTok}: million tokens per episode. \emph{Rank}: mean outcome rank within a block. \textbf{Bold} and \underline{underlined}: best and second best Outc.\ and Proc.\ per capability within a block. Rows are ordered by mean rank over all five models, worst first.}
  \label{tab:appendix-results}
  \setlength{\tabcolsep}{3pt}
  \renewcommand{\arraystretch}{1.12}
  \resizebox{\textwidth}{!}{%
  \begin{tabular}{lI cccI cccI cccI cccI c}
    \Xhline{1.2pt}
    \multirow{2}{*}{\textbf{Harness}}
      & \multicolumn{3}{cI}{\textbf{\ctool}}
      & \multicolumn{3}{cI}{\textbf{\cmem}}
      & \multicolumn{3}{cI}{\textbf{\cplan}}
      & \multicolumn{3}{cI}{\textbf{\cweb}}
      & \multirow{2}{*}{\textbf{Rank}} \\
      & \textbf{Outc.}$\uparrow$  & \textbf{Proc.}$\uparrow$  & \textbf{MTok}$\downarrow$  & \textbf{Outc.}$\uparrow$  & \textbf{Proc.}$\uparrow$$^\dagger$  & \textbf{MTok}$\downarrow$  & \textbf{Outc.}$\uparrow$  & \textbf{Proc.}$\uparrow$  & \textbf{MTok}$\downarrow$  & \textbf{Outc.}$\uparrow$  & \textbf{Proc.}$\uparrow$  & \textbf{MTok}$\downarrow$ & \\
    \Xhline{0.5pt}
    \rowcolor{mhBlock!70}
    \multicolumn{14}{c}{\color{mhNavy}\textbf{\mlightning}} \\
    \baseline{ZeroClaw}      & \meanstd{0.05}{0.02} & \meanstd{0.23}{0.04} & 0.53      & \meanstd{0.55}{0.09} & \meanstd{0.20}{0.07} & 0.88      & \meanstd{0.01}{0.02} & \meanstd{0.17}{0.10} & 3.23      & \meanstd{0.00}{0.00} & \meanstd{0.08}{0.02} & 9.26      & 4.50 \\
    \baseline{OpenClaw}      & \meanstd{0.06}{0.01} & \meanstd{0.28}{0.07} & 0.65      & \meanstd{0.59}{0.05} & \meanstd{0.24}{0.05} & 1.35      & \meanstd{0.00}{0.00} & \meanstd{0.15}{0.01} & 5.00      & \meanstd{0.00}{0.00} & \meanstd{0.08}{0.02} & 12.22      & 3.75 \\
    \baseline{Hermes}      & \secondresult{\meanstd{0.07}{0.00}} & \firstresult{\meanstd{0.36}{0.02}} & 0.23      & \meanstd{0.62}{0.05} & \secondresult{\meanstd{0.31}{0.05}} & 0.85      & \firstresult{\meanstd{0.02}{0.04}} & \firstresult{\meanstd{0.40}{0.08}} & 1.87      & \meanstd{0.00}{0.00} & \meanstd{0.07}{0.03} & 11.54      & 2.50 \\
    \baseline{Codex}      & \firstresult{\meanstd{0.08}{0.01}} & \secondresult{\meanstd{0.33}{0.03}} & 0.60      & \secondresult{\meanstd{0.68}{0.05}} & \firstresult{\meanstd{0.41}{0.05}} & 0.72      & \secondresult{\meanstd{0.02}{0.02}} & \secondresult{\meanstd{0.25}{0.09}} & 2.95      & \firstresult{\meanstd{0.04}{0.04}} & \secondresult{\meanstd{0.08}{0.01}} & 8.01      & 1.50 \\
    \baseline{Claude Code}      & \meanstd{0.06}{0.01} & \meanstd{0.29}{0.03} & 0.35      & \firstresult{\meanstd{0.71}{0.05}} & \meanstd{0.29}{0.05} & 1.20      & \meanstd{0.00}{0.00} & \meanstd{0.22}{0.09} & 3.53      & \secondresult{\meanstd{0.01}{0.02}} & \firstresult{\meanstd{0.12}{0.05}} & 9.76      & 2.75 \\
    \Xhline{0.5pt}
    \rowcolor{mhBlock!70}
    \multicolumn{14}{c}{\color{mhNavy}\textbf{\mglm}} \\
    \baseline{ZeroClaw}      & \meanstd{0.03}{0.01} & \secondresult{\meanstd{0.31}{0.00}} & 0.41      & \secondresult{\meanstd{0.48}{0.12}} & \secondresult{\meanstd{0.44}{0.05}} & 0.23      & \meanstd{0.00}{0.00} & \meanstd{0.10}{0.06} & 1.55      & \meanstd{0.15}{0.10} & \meanstd{0.31}{0.06} & 2.65      & 3.25 \\
    \baseline{OpenClaw}      & \meanstd{0.02}{0.00} & \meanstd{0.29}{0.02} & 0.58      & \meanstd{0.43}{0.04} & \meanstd{0.36}{0.02} & 0.59      & \meanstd{0.00}{0.00} & \meanstd{0.11}{0.04} & 1.63      & \meanstd{0.00}{0.00} & \meanstd{0.27}{0.06} & 3.71      & 4.25 \\
    \baseline{Hermes}      & \firstresult{\meanstd{0.04}{0.02}} & \firstresult{\meanstd{0.32}{0.04}} & 0.41      & \meanstd{0.35}{0.08} & \firstresult{\meanstd{0.48}{0.06}} & 0.28      & \secondresult{\meanstd{0.01}{0.02}} & \firstresult{\meanstd{0.33}{0.04}} & 1.03      & \meanstd{0.06}{0.05} & \meanstd{0.14}{0.03} & 1.10      & 3.00 \\
    \baseline{Codex}      & \secondresult{\meanstd{0.04}{0.00}} & \meanstd{0.28}{0.03} & 0.61      & \firstresult{\meanstd{0.54}{0.05}} & \meanstd{0.38}{0.09} & 1.55      & \meanstd{0.01}{0.02} & \secondresult{\meanstd{0.28}{0.06}} & 2.64      & \secondresult{\meanstd{0.21}{0.15}} & \firstresult{\meanstd{0.34}{0.11}} & 8.30      & 2.00 \\
    \baseline{Claude Code}      & \meanstd{0.02}{0.00} & \meanstd{0.28}{0.03} & 0.87      & \meanstd{0.48}{0.08} & \meanstd{0.34}{0.08} & 1.28      & \firstresult{\meanstd{0.06}{0.05}} & \meanstd{0.23}{0.14} & 2.80      & \firstresult{\meanstd{0.25}{0.04}} & \secondresult{\meanstd{0.32}{0.04}} & 4.40      & 2.50 \\
    \Xhline{1.2pt}
  \end{tabular}}
\end{table*}

\textbf{Ranking stability.} \Cref{tab:ranking-stability} gives the leader, its probability of ranking
first and its gap to the runner-up for every model and environment.

\begin{table*}[t]
  \centering
  \caption{\small\textbf{Ranking stability.} Leader per model and environment, the bootstrap probability $P$ that it ranks first and its gap to the runner-up with a 95\% interval (paired bootstrap over tasks, $20{,}000$ resamples). Bold marks an interval that excludes zero.}
  \label{tab:ranking-stability}
  \small
  \setlength{\tabcolsep}{4pt}
  \resizebox{\textwidth}{!}{%
  \begin{tabular}{lI lcc I lcc}
    \Xhline{1.2pt}
    \multirow{2}{*}{\textbf{Model}} & \multicolumn{3}{cI}{\textbf{\ctool}} & \multicolumn{3}{c}{\textbf{\cmem}} \\
    & Leader & $P$ & Gap [95\% CI] & Leader & $P$ & Gap [95\% CI] \\
    \Xhline{0.5pt}
    \mgemma & Hermes & 0.95 & $+$0.04 [0.00, $+$0.09] & Claude Code & 0.61 & $+$0.01 [$-$0.07, $+$0.10] \\
    \mqwen & Hermes & 0.91 & $+$0.02 [$-$0.01, $+$0.04] & Codex & 0.47 & 0.00 [$-$0.10, $+$0.10] \\
    \mgptoss & Claude Code & 0.94 & $+$0.02 [0.00, $+$0.04] & Claude Code & 0.97 & \textbf{$+$0.12 [$+$0.03, $+$0.23]} \\
    \mlightning & Codex & 0.80 & $+$0.01 [$-$0.01, $+$0.03] & Claude Code & 0.68 & $+$0.03 [$-$0.06, $+$0.12] \\
    \mglm & Hermes & 0.59 & 0.00 [$-$0.01, $+$0.02] & Codex & 0.66 & $+$0.06 [$-$0.09, $+$0.20] \\
    \Xhline{1.2pt}
  \end{tabular}}
  \par\vspace{6pt}
  \resizebox{\textwidth}{!}{%
  \begin{tabular}{lI lcc I lcc}
    \Xhline{1.2pt}
    \multirow{2}{*}{\textbf{Model}} & \multicolumn{3}{cI}{\textbf{\cplan}} & \multicolumn{3}{c}{\textbf{\cweb}} \\
    & Leader & $P$ & Gap [95\% CI] & Leader & $P$ & Gap [95\% CI] \\
    \Xhline{0.5pt}
    \mgemma & Claude Code & 0.78 & $+$0.02 [$-$0.02, $+$0.08] & Codex & 0.61 & $+$0.01 [$-$0.07, $+$0.11] \\
    \mqwen & OpenClaw & 0.39 & 0.00 [$-$0.07, $+$0.07] & Claude Code & 0.59 & $+$0.01 [$-$0.10, $+$0.12] \\
    \mgptoss & Claude Code & 0.47 & $+$0.01 [$-$0.06, $+$0.08] & Claude Code & 1.00 & \textbf{$+$0.24 [$+$0.11, $+$0.38]} \\
    \mlightning & Hermes & 0.41 & 0.00 [$-$0.04, $+$0.04] & Codex & 0.84 & $+$0.03 [$-$0.03, $+$0.08] \\
    \mglm & Claude Code & 0.93 & $+$0.05 [$-$0.01, $+$0.12] & Claude Code & 0.76 & $+$0.04 [$-$0.06, $+$0.15] \\
    \Xhline{1.2pt}
  \end{tabular}}
\end{table*}

\textbf{Token use.} Token use per episode, which counts the history resent with every call, differs by
a median $2.6\times$ across harnesses within a model--environment pair. The highest-scoring harness
also uses the fewest tokens in only $3$ of $20$ pairs. Under \mqwen, Hermes reaches $0.78$ on \cmem
with $0.41$ million tokens per episode, while Codex reaches $0.86$ with $0.62$ million.

\subsection{Execution outcomes and scoring analyses}
\label{sec:appendix-execution}

\textbf{Episode endings.} \Cref{fig:task-diagnostics}a groups the episodes of \mqwen by how they end.
Categories are decided in a fixed order: explicit termination evidence first, then the state of the
answer or evaluated output, and an error earlier in a run that continues does not decide the category.
Positive score includes partial credit on \ctool. Scored zero is a scorable output or application state
that earns no credit. No valid output is a required output that is missing, unparseable or in the wrong
format. Timeout is the one-hour wall-clock limit. Limit or error covers context and request limits and
other terminal errors. Unresolved episodes lack the evidence to decide.

\textbf{Answer rate and score.} \Cref{tab:answer-rate} gives the values behind
\Cref{fig:harness-comparison}b, counting an answer wherever the agent gave it. Across the five models,
$15\%$ of \ctool episodes and $13\%$ of \cmem episodes end without an answer, while $96\%$ of \cplan
episodes produce output files. Because a harness that answers more may answer harder items, scores are
compared on the $n$ matched task--seed pairs that every harness answered. On these pairs, scores vary
little across harnesses except for \mgptoss and \mglm on \cmem, with ranges of $0.56$--$0.85$ and
$0.44$--$0.75$.

\textbf{Answer time.} The bottom rows of \Cref{tab:answer-rate} recount answers retrospectively at
earlier cutoffs, to show how the answer-rate gap depends on the one-hour limit. Each answer's time is
the last write of the answer file in the saved trajectory, or the episode's end for reply-only answers
and files without a usable timestamp. With more time, the \ctool answer-rate spread narrows from $0.49$
to $0.29$, while the outcome spreads stay at $0.04$--$0.05$ on \ctool and $0.24$--$0.25$ on \cmem.

\begin{table}[t]
  \centering
  \caption{\small\textbf{Answer rate and score.} Top: minimum--maximum across harnesses of the answer
  rate over all episodes and mean score on the $n$ matched task--seed pairs every harness answered.
  Score is the task outcome (\Cref{sec:scoring}). Bottom: maximum-minus-minimum spreads over
  all episodes, averaged over five models, at retrospective answer-time cutoffs $T$ with a
  $30$-second clock tolerance. Recorded answers after the cutoff count as missing and score zero.
  Answers given only in the reply count.}
  \label{tab:answer-rate}
  \footnotesize
  \setlength{\tabcolsep}{4pt}
  \begin{tabular}{lI ccrI ccr}
    \Xhline{1.2pt}
    & \multicolumn{3}{cI}{\textbf{\ctool}} & \multicolumn{3}{c}{\textbf{\cmem}} \\
    \textbf{Model} & Answer rate & Score (shared) & $n$ & Answer rate & Score (shared) & $n$ \\
    \Xhline{0.5pt}
    \mgemma & 0.50--0.95 & 0.12--0.16 & 28 & 0.57--0.91 & 0.87--0.90 & 30 \\
    \mqwen & 0.46--1.00 & 0.10--0.16 & 39 & 0.81--1.00 & 0.82--0.88 & 56 \\
    \mgptoss & 0.80--0.98 & 0.13--0.17 & 57 & 0.59--1.00 & 0.56--0.85 & 27 \\
    \mlightning & 0.59--0.81 & 0.08--0.12 & 28 & 0.77--0.99 & 0.69--0.73 & 45 \\
    \mglm & 0.91--0.98 & 0.02--0.04 & 72 & 0.80--0.96 & 0.44--0.75 & 32 \\
    \Xhline{0.5pt}
    \textbf{Cutoff} & Rate spread & Outcome spread & & Rate spread & Outcome spread & \\
    \Xhline{0.5pt}
    $T=15$ min & 0.49 & 0.04 &  & 0.28 & 0.25 &  \\
    $T=30$ min & 0.42 & 0.05 &  & 0.27 & 0.24 &  \\
    $T=45$ min & 0.37 & 0.05 &  & 0.27 & 0.25 &  \\
    $T=60$ min & 0.29 & 0.04 &  & 0.26 & 0.25 &  \\
    \Xhline{1.2pt}
  \end{tabular}
\end{table}

\textbf{File-only scoring.} On \ctool and \cmem, $14\%$ and $17\%$ of answers appear only in the final
reply (\Cref{fig:delivery}). Scoring the answer file alone would lower the \ctool scores of \mgemma
from $0.06$--$0.13$ to $0.00$--$0.05$ and the \cmem scores of \mlightning from $0.55$--$0.71$ to
$0.20$--$0.43$. It would also change the harness rankings, with Kendall $\tau$ as low as $0.40$ and a
different leader under \mgemma on \cmem.

\begin{figure}[!h]
  \centering
  \includegraphics[width=\textwidth]{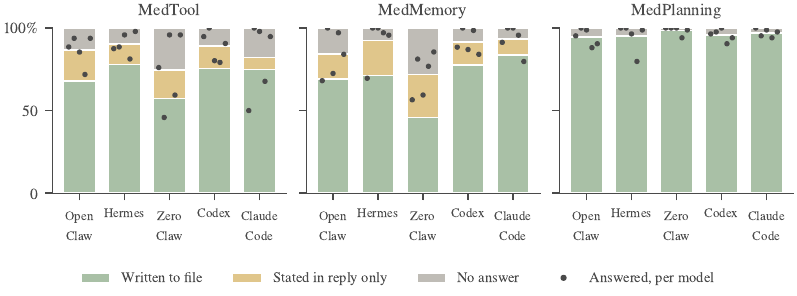}
  \caption{\small\textbf{Where the answer ends up.} Share of episodes whose answer was written to the
  named file, stated only in the final reply or absent, averaged over five models. Dots show each
  model's answer rate.}
  \label{fig:delivery}
\end{figure}

\FloatBarrier
\section{Splitting the harness variance into answer rate and answer score}
\label{sec:appendix-proofs}

When a task asks for an answer, an episode can end without one. For a harness $h$, with the model and
environment fixed, let $A = 1$ if an episode ends with an answer, wherever the agent gives it, and
$Y = 0$ if $A = 0$. The answer rate is $r_h = \Pr(A = 1 \mid h)$ and the mean score among answered episodes is
$q_h = \mathbb{E}[Y \mid A = 1, h]$, so that $\bar y_h = r_h q_h$. With $a_h = r_h - \bar r$ and
$b_h = q_h - \bar q$ the deviations from the means over harnesses, expanding
$\bar y_h = \bar r \bar q + \bar q a_h + \bar r b_h + a_h b_h$ gives
\begin{align*}
\operatorname{Var}_h \bar y_h &= V_r + V_q + R, \qquad V_r = \bar q^{2}\operatorname{Var}(a),
\qquad V_q = \bar r^{2}\operatorname{Var}(b),\\
R &= 2\bar r\bar q\operatorname{Cov}(a, b) + \operatorname{Var}(ab) + 2\bar q\operatorname{Cov}(a, ab)
  + 2\bar r\operatorname{Cov}(b, ab),
\end{align*}
with moments taken over harnesses. $V_r$ is the variance the answer rate alone would produce, and $V_q$
the variance the answer score alone would produce. The first term of $R$ vanishes when answer rate and
mean answer score are uncorrelated across harnesses, and the others are of third and fourth order in
the deviations. The covariance is negative when harnesses that answer more score lower on what they
answer, as when they answer harder items.

Over the $10$ model--environment pairs on \ctool and \cmem, the first-order answer-rate share
$V_r/(V_r + V_q)$ has a median of $61\%$ and is the larger in $7$ of the $10$ pairs. The remainder
$|R|$ is below a third of $\operatorname{Var}_h \bar y_h$ in $6$ of the $10$ pairs. It is larger where
the outcome variance is small (\ctool under \mqwen and \mgptoss) and where answer rate and mean answer
score move in opposite directions (\mlightning and \mglm on \cmem). For \mqwen on \ctool, $V_r + V_q$
is $3.05$ times $\operatorname{Var}_h \bar y_h$, so the negative remainder cancels about two thirds of
them. Because $q_h$ averages over the items $h$ answered,
\Cref{sec:delivery} compares scores on matched task--seed pairs every harness answered.

\clearpage
\section{\lgharness: implementation and ablation details}
\label{sec:appendix-mhloop}
\suppressfloats[t]

This section gives the configuration and comparison protocol behind \Cref{tab:final-switches}, the
intervals and task diagnostics of its outcome changes and the execution evidence that explains them.

\subsection{Implementation and comparison protocol}
\label{sec:appendix-mhloop-implementation}

\Cref{tab:mhlab-config} lists the configuration of the final matrix. Context management summarizes
older turns, as the product harnesses in \Cref{sec:appendix-modules} do. The input budget sits below the
native windows so that summarization fires on long trajectories. \mqwen writes long reasoning, so it gets
a larger window and output cap than \mgptoss while keeping the same input budget.

\begin{table}[h]
  \centering
  \caption{\small\textbf{\lgharness configuration of the final matrix.}}
  \label{tab:mhlab-config}
  \footnotesize
  \setlength{\tabcolsep}{4pt}
  \begin{tabular}{l>{\raggedright\arraybackslash}p{10.2cm}}
    \Xhline{1.2pt}
    \textbf{Setting} & \textbf{Value} \\
    \Xhline{0.5pt}
    Matrix & two models, four conditions, $107$ tasks and three seeds, for $2{,}568$ scheduled episodes \\
    Input budget & $65{,}536$ tokens per request for both models, with $5\%$ reserved for token-estimation error \\
    Window and output cap & \mqwen: $131$K window, $65$K output cap. \mgptoss: $80$K window, $16$K output cap \\
    File tools & seven workspace tools for reading, writing, editing, listing and searching files and running shell commands, with workspace confinement and read-before-write checks \\
    Browser tools & fourteen browser tools \\
    Execution limits & tool results capped at $24$K characters, loop detection, a $500$-step limit and a one-hour deadline \\
    Context management & summarization when estimated input reaches $80\%$ of the input budget after the $5\%$ reserve \\
    Planning & an \texttt{update\_plan} tool, planning instructions, reminders and the plan shown on every turn \\
    Tool exposure off & task tools reachable only through \texttt{tool\_search}, \texttt{tool\_describe} and \texttt{tool\_call}, and a tool must be discovered before it is called. The planning tool stays directly exposed \\
    \Xhline{1.2pt}
  \end{tabular}
\end{table}

\paragraph{Pairing and exclusions.} Each comparison pairs the same task and seed under \emph{full} and
the switched condition, and a pair is dropped when either episode is excluded. Excluded are provider
failures ($23$ \cweb, $2$ \cplan and $1$ \ctool episodes of \mgptoss whose responses vLLM could not parse,
and $1$ \cplan episode of \mqwen) and two \mqwen planning-off pathology episodes stopped by the
$200$\,GiB memory cap when a shell script loaded a whole slide, which left no result. Since some of
these failures stem from the model's own output or actions, we also count every excluded episode as a
failure. All statistically significant outcome effects in \Cref{tab:final-switches} then retain their direction
and significance. The largest
shifts are on \cweb for \mgptoss, where the three changes move from $-0.17$, $-0.18$ and $-0.15$ to
$-0.14$, $-0.17$ and $-0.12$.

\subsection{Outcome intervals and task diagnostics}
\label{sec:appendix-mhloop-results}

\Cref{tab:final-intervals} gives the 95\% intervals of the outcome changes in \Cref{tab:final-switches},
next to \cplan recall and the share of \cweb process checks passed. These two measures show one effect
the pass rates hide. Without context management, \cplan recall falls by $0.29$ for \mqwen and $0.16$ for
\mgptoss although the outcome barely moves, because the pass threshold is rarely met in any condition.
For the mechanism comparisons, we report environment-based task diagnostics.

\begin{table*}[t]
  \centering
  \caption{\small\textbf{Outcome and process changes behind \Cref{tab:final-switches}.} \emph{full}: mean $\pm$ SD over seeds. Other rows: paired change in outcome with a 95\% bootstrap interval over tasks, and paired change ($\pm$ SD of the per-seed change) in \cplan recall on the $17$ tasks with a reference set and in the share of \cweb process checks passed.}
  \label{tab:final-intervals}
  \small\setlength{\tabcolsep}{4pt}\renewcommand{\arraystretch}{1.1}
  \begin{tabularx}{\textwidth}{lI YI Y}
    \Xhline{1.2pt}
    \multirow{2}{*}{\textbf{Condition}} & \multicolumn{1}{cI}{\textbf{\ctool}} & \multicolumn{1}{c}{\textbf{\cmem}} \\
    & Outc.\ [95\% CI] & Outc.\ [95\% CI] \\
    \Xhline{0.5pt}
    \rowcolor{mhBlock!70}
    \multicolumn{3}{c}{\color{mhNavy}\textbf{\mqwen}} \\
    \emph{full} & \meanstd{0.10}{0.02} & \meanstd{0.88}{0.03} \\
    $-$ context management & $+$0.01 [$-$0.01, $+$0.04] & $-$0.12 [$-$0.23, $-$0.01] \\
    $-$ planning & $-$0.02 [$-$0.04, 0.00] & $-$0.04 [$-$0.12, $+$0.01] \\
    $-$ tool exposure & 0.00 [$-$0.02, $+$0.03] & $-$0.14 [$-$0.23, $-$0.07] \\
    \Xhline{0.5pt}
    \rowcolor{mhBlock!70}
    \multicolumn{3}{c}{\color{mhNavy}\textbf{\mgptoss}} \\
    \emph{full} & \meanstd{0.14}{0.03} & \meanstd{0.72}{0.03} \\
    $-$ context management & $-$0.01 [$-$0.04, $+$0.02] & $+$0.06 [$-$0.09, $+$0.20] \\
    $-$ planning & $-$0.02 [$-$0.04, $+$0.01] & $-$0.01 [$-$0.10, $+$0.09] \\
    $-$ tool exposure & $-$0.04 [$-$0.07, $-$0.02] & $-$0.06 [$-$0.14, $+$0.03] \\
    \Xhline{1.2pt}
  \end{tabularx}
  \par\vspace{6pt}
  \begin{tabularx}{\textwidth}{lI Y cI Y c}
    \Xhline{1.2pt}
    \multirow{2}{*}{\textbf{Condition}} & \multicolumn{2}{cI}{\textbf{\cplan}} & \multicolumn{2}{c}{\textbf{\cweb}} \\
    & Outc.\ [95\% CI] & Recall & Outc.\ [95\% CI] & Checks \\
    \Xhline{0.5pt}
    \rowcolor{mhBlock!70}
    \multicolumn{5}{c}{\color{mhNavy}\textbf{\mqwen}} \\
    \emph{full} & \meanstd{0.12}{0.05} & \meanstd{0.54}{0.03} & \meanstd{0.44}{0.02} & \meanstd{0.44}{0.07} \\
    $-$ context management & $-$0.04 [$-$0.12, $+$0.02] & \meanstd{$-$0.29}{0.01} & $-$0.39 [$-$0.57, $-$0.22] & \meanstd{$-$0.17}{0.11} \\
    $-$ planning & $+$0.01 [$-$0.07, $+$0.10] & \meanstd{$-$0.04}{0.11} & 0.00 [$-$0.08, $+$0.08] & \meanstd{$-$0.07}{0.08} \\
    $-$ tool exposure & $+$0.01 [$-$0.06, $+$0.07] & \meanstd{$-$0.12}{0.14} & 0.00 [$-$0.06, $+$0.07] & \meanstd{$-$0.02}{0.05} \\
    \Xhline{0.5pt}
    \rowcolor{mhBlock!70}
    \multicolumn{5}{c}{\color{mhNavy}\textbf{\mgptoss}} \\
    \emph{full} & \meanstd{0.06}{0.02} & \meanstd{0.40}{0.03} & \meanstd{0.18}{0.08} & \meanstd{0.33}{0.04} \\
    $-$ context management & $+$0.02 [$-$0.02, $+$0.07] & \meanstd{$-$0.16}{0.06} & $-$0.17 [$-$0.28, $-$0.07] & \meanstd{$-$0.18}{0.17} \\
    $-$ planning & $+$0.02 [$-$0.02, $+$0.07] & \meanstd{$-$0.08}{0.03} & $-$0.18 [$-$0.29, $-$0.08] & \meanstd{$-$0.19}{0.05} \\
    $-$ tool exposure & $-$0.04 [$-$0.10, 0.00] & \meanstd{$-$0.04}{0.06} & $-$0.15 [$-$0.26, $-$0.06] & \meanstd{$-$0.10}{0.05} \\
    \Xhline{1.2pt}
  \end{tabularx}
\end{table*}

\subsection{Execution evidence}
\label{sec:appendix-mhloop-evidence}

\Cref{tab:final-mechanisms} gives, for each mechanism, the execution measure behind its outcome change.

\paragraph{Context management.} Without summaries, traces outgrow the input budget. Overflows rise
from almost none to $67$ of $72$ \cweb episodes for \mqwen and $48$ of $65$ for \mgptoss, and to $57$ of
$84$ \cplan episodes for \mqwen.

\paragraph{Planning.} Without a plan, \mgptoss ends \cweb episodes much sooner, with median turns
falling from $176.5$ to $63$ on the paired episodes. \mqwen instead leaves more \ctool work unfinished,
with missing outputs rising from $24$ to $47$ of $96$ and timeouts from $22$ to $44$.

\paragraph{Tool exposure.} With tools behind search, models still name hidden tools directly, \mqwen in
$95$ of $96$ \ctool episodes and \mgptoss in $86$. Between $98\%$ and $100\%$ of these rejected calls
follow a successful search, so the models find the tool but then call it through the wrong interface.
The discovery gate itself (\texttt{tool\_call} before a search) fires at most once per episode.

\begin{table*}[t]
  \centering
  \caption{\small\textbf{Execution evidence behind \Cref{tab:final-switches}.} Each cell goes from \emph{full} to the switched condition. Counts are over the valid episodes of each condition, and median turns over the task--seed pairs valid in both. Where a condition lost episodes to exclusion, counts show their denominators. \cweb is scored on application state and has no file to submit. With tools hidden, the columns give the mean number of rejected direct calls per episode and the episodes with at least one.}
  \label{tab:final-mechanisms}
  \footnotesize\setlength{\tabcolsep}{3pt}
  \begin{tabular}{lI cI cccI cc}
    \Xhline{1.2pt}
    & \makecell{$-$ context\\management} & \multicolumn{3}{cI}{$-$ planning} & \multicolumn{2}{c}{$-$ tool exposure} \\
    \textbf{Environment} & Overflows & Median turns & \makecell{No valid\\submission} & Timeouts & \makecell{Rejected calls\\per episode} & \makecell{Episodes with\\a rejected call} \\
    \Xhline{0.5pt}
    \multicolumn{7}{l}{\mqwen} \\
    \ctool & 0 $\to$ 0 & 10 $\to$ 9 & 24 $\to$ 47 & 22 $\to$ 44 & 18 & 95/96 \\
    \cmem & 0 $\to$ 8 & 9 $\to$ 6 & 1 $\to$ 0 & 1 $\to$ 0 & 43 & 66/69 \\
    \cplan & 2 $\to$ 57 & 37 $\to$ 38 & 3/84 $\to$ 3/81 & 8/84 $\to$ 16/81 & 79 & 84/84 \\
    \cweb & 0 $\to$ 67 & 99.5 $\to$ 78.5 & -- & 3 $\to$ 1 & 31 & 47/72 \\
    \Xhline{0.5pt}
    \multicolumn{7}{l}{\mgptoss} \\
    \ctool & 0 $\to$ 1 & 18 $\to$ 10 & 6/96 $\to$ 10/95 & 0/96 $\to$ 0/95 & 12 & 86/96 \\
    \cmem & 0 $\to$ 4 & 15 $\to$ 16 & 0 $\to$ 5 & 0 $\to$ 0 & 6 & 57/69 \\
    \cplan & 0/84 $\to$ 22/83 & 31 $\to$ 29 & 1/84 $\to$ 0/83 & 1/84 $\to$ 0/83 & 34 & 82/84 \\
    \cweb & 0/66 $\to$ 48/65 & 176.5 $\to$ 63 & -- & 6/66 $\to$ 1/66 & 64 & 68/68 \\
    \Xhline{1.2pt}
  \end{tabular}
\end{table*}

\FloatBarrier
\section{Module activation and switches on the product harnesses}
\label{sec:appendix-modules}

\textbf{Module activation.} Before a mechanism inside a product harness can be ablated, it has to fire.
We therefore counted how often each advertised module was used, in the archived episodes of the main
study and of a pilot with Deep Agents, LangChain's agent harness. Hermes, OpenClaw and Claude Code expose skill tools in every
episode, and none of the $321$ episodes per model called one. Delegation was invoked in $0$--$5$ episodes
per model, for instance in $4$ of $32$ \ctool episodes by Deep Agents with \mqwen. Hermes's loop guards
intervened in no episode, and neither Hermes nor OpenClaw forwards its reasoning-effort setting to a
custom vLLM provider. Summarization fires in $2$--$5\%$ of episodes at the native windows.

\textbf{Detectability.} A mechanism that fires this rarely cannot be detected by switching it off. With
shared sampling randomness, an episode in which it never fires ends the same way with or without it, so
for a firing rate $p$ the paired effect is at most $p$, and at most $\sqrt{p/(1-p)}$ standard deviations
of the paired difference. With $69$--$96$ paired episodes per environment, the expected $z$-statistic
of switching summarization off is then at most $1.19$ ($p = 0.02$, $n = 69$) to $2.25$ ($p = 0.05$,
$n = 96$), even if every firing turned a success into a failure. Our serving stack does not repeat a
trajectory exactly under a fixed seed, which only makes such a mechanism harder to detect. The ablation
in \Cref{sec:components} therefore uses an input budget at which context management fires. The switches
in \Cref{tab:appendix-switches} were run where each mechanism is reachable, as paired comparisons on the
same tasks with every other setting held fixed.

\begin{table}[h]
  \centering
  \caption{\small\textbf{Paired switches on the product harnesses.} Outcome and context overflows go from
  the default to the switched setting. \emph{Pairs} counts paired episodes, over three seeds where it is $69$
  and over one seed otherwise. \emph{Tool search} is
  Hermes's \texttt{tool\_search}/\texttt{tool\_call} layer. With it off, \mgptoss browses in $24/24$ \cweb
  episodes instead of $3/24$. The \emph{submission check} adds one reminder round when the answer file is
  missing.}
  \label{tab:appendix-switches}
  \footnotesize
  \setlength{\tabcolsep}{3pt}
  \resizebox{\textwidth}{!}{%
  \begin{tabular}{llllcc}
    \Xhline{1.2pt}
    \textbf{Module} & \textbf{Harness $\times$ model} & \textbf{Environment, window} & \textbf{Outcome} & \textbf{Ovf.} & \textbf{Pairs} \\
    \Xhline{0.5pt}
    Summarization & OpenClaw $\times$ \mqwen & \cmem, 80K & $0.62 \to 0.44$ & $8 \to 25$ & 69 \\
    Summarization & Deep Agents $\times$ \mqwen & \cmem, 80K & $0.52 \to 0.38$ & $0 \to 10$ & 69 \\
    Summarization & Deep Agents $\times$ \mqwen & \cmem, 64K & $0.54 \to 0.45$ & $0 \to 15$ & 69 \\
    Summarization & Hermes $\times$ \mqwen & \cmem, 80K & $0.83 \to 0.83$ & $0 \to 0$ & 23 \\
    Summarization & Deep Agents $\times$ \mgptoss & \cmem, 80K & $0.52 \to 0.48$ & $0 \to 4$ & 69 \\
    Tool-output truncation & Hermes $\times$ \mqwen & \cmem, native & $0.78 \to 0.78$ & -- & 23 \\
    Tool-output truncation & Hermes $\times$ \mqwen & \cplan, native & $0.11 \to 0.06$ & -- & 18 \\
    Tool-output truncation & Hermes $\times$ \mqwen & \ctool, native & $0.16 \to 0.14$ & -- & 32 \\
    Tool search (off) & Hermes $\times$ \mqwen & \cweb, native & $0.17 \to 0.50$ & -- & 12 \\
    Tool search (off) & Hermes $\times$ \mgptoss & \cweb, native & $0.00 \to 0.00$ & -- & 24 \\
    Submission check (on) & ZeroClaw $\times$ \mqwen & \cmem, native & $0.42 \to 0.51$ & -- & 69 \\
    Submission check (on) & Hermes $\times$ \mgptoss & \cmem, native & $0.65 \to 0.52$ & -- & 23 \\
    \Xhline{1.2pt}
  \end{tabular}}
\end{table}

\textbf{Switches on the product harnesses.} The results match RQ3 (\Cref{sec:components}). Summarization
matters where a trace would otherwise outgrow the window and the harness has no other fallback: OpenClaw
and Deep Agents end such an episode on the provider's overflow error, Hermes trims the history and
continues, and \mgptoss rarely reaches the window at all. Tool search decides whether the browser tools
are visible. With it off, \mgptoss browses in every episode but still completes none, and \mqwen rises
from $0.17$ to $0.50$. The submission check turns a
finished episode without a file into a filed answer only when the model already had an answer.
\mqwen's stated answers are right, so its score rises. The four answers \mgptoss filed after a reminder
are all wrong, so the check changes none of its scores, and its drop from $0.65$ to $0.52$ comes from
episodes that already differed before any reminder. Truncating tool outputs, which Hermes does in nearly
every episode, has no clear effect.

\end{document}